\documentclass{article}

\PassOptionsToPackage{numbers}{natbib}
\usepackage[final]{neurips_2023}

\usepackage[utf8]{inputenc} 
\usepackage[T1]{fontenc}    
\usepackage{hyperref}       
\usepackage{url}            
\usepackage{booktabs}       
\usepackage{amsfonts}       
\usepackage{nicefrac}       
\usepackage{microtype}      
\usepackage{xcolor}         
\usepackage{wrapfig2}
\usepackage{amsmath, amssymb, amsthm}
\usepackage{amsfonts}
\usepackage{pifont}
\usepackage{graphicx}
\usepackage{caption}
\usepackage{subcaption}
\usepackage{enumerate}

\usepackage{tikz}
\usetikzlibrary{positioning, shapes, bayesnet}

\newcommand{\R}{\mathbb{R}}
\newcommand{\E}[2]{\ensuremath{\mathbb{E}_{#2}\left[#1\right]}}
\newcommand{\p}{\ensuremath{p}}
\newcommand{\KL}[2]{\ensuremath{D_{KL}(#1\mid\mid#2 )}}
\newcommand{\rnorm}[1]{\ensuremath{\lvert\lvert{#1}\rvert\rvert_{C_\theta(t)}^2}}
\newcommand{\modifiedrnorm}[1]{\ensuremath{\lvert\lvert{#1}\rvert\rvert_{C_{V(\nu^{(k)},\phi)}(t)}^2}}
\definecolor{mplblue}{HTML}{1f77b4}
\definecolor{mplorange}{HTML}{ff7f0e}
\definecolor{mplgreen}{HTML}{2ca02c}
\newtheorem{lemma}{Lemma}
\newtheorem{theorem}{Theorem}

\title{Amortized Reparametrization: Efficient and Scalable Variational Inference for Latent SDEs}

\author{%
  Kevin~Course\\
  University of Toronto\\
  \texttt{kevin.course@mail.utoronto.ca} \\
  \And
  Prasanth~B.~Nair \\
  University of Toronto\\
  \texttt{prasanth.nair@utoronto.ca} \\
}

\begin{document}

\maketitle

\begin{abstract}
  We consider the problem of inferring latent stochastic 
  differential equations (SDEs)
  with a time and memory cost that scales independently with  
  the amount of data,
  the total length of the time series, 
  and the stiffness of the approximate differential equations.
  This is in stark contrast to typical methods for inferring latent 
  differential equations which, despite their constant memory cost, 
  have a time complexity that is heavily dependent on the 
  stiffness of the approximate differential equation.
  We achieve this computational advancement by removing the 
  need to solve differential equations when approximating 
  gradients using a novel amortization strategy
  coupled with a recently derived reparametrization of expectations
  under linear SDEs.
  We show that, in practice, this allows us to achieve similar performance 
  to methods based on adjoint sensitivities with more than
  an order of magnitude fewer evaluations of the model in training.
\end{abstract}

\section{Introduction}
Recent years have seen the rise of continuous time 
models for dynamical system modeling~\cite{chen_neural_2018}. 
As compared to traditional autoregressive style models~\cite{hochreiter_long_1997},
continuous time models are useful because 
they can deal with non-evenly spaced observations, 
they enable multilevel/hierarchical and adaptive prediction schemes, 
and because physics is (mostly) done in continuous time.
For example, recent developments in inferring continuous time 
differential equations from data has been met by 
a flurry of work in 
endowing models with physics informed priors~\cite{greydanus_hamiltonian_2019,cranmer_lagrangian_2020,course_weak_2020}.

Despite their advantages, continuous time models remain
significantly more computationally challenging to train than their 
autoregressive counterparts due to their 
reliance on adjoint methods for estimating gradients.
Adjoint methods introduce a significant 
computational burden in training 
because they require solving a pair of initial 
value problems to estimate gradients. 
Solving such initial value problems as a part of 
an iterative optimization procedure is computationally 
demanding for the following reasons:
\begin{enumerate}[(i)]
  \item Gradient based updates to models of differential equations 
    can cause them to become extremely stiff. This will have 
    the effect of causing the cost per iteration to explode
    mid-optimization.
  \item With the exception of parareal methods~\cite{lorin_derivation_2020},
    differential equation solvers are fundamentally 
    {\em iterative sequential} methods. This makes them 
    poorly suited to being parallelized on modern parallel 
    computing hardware. 
\end{enumerate}

In accordance with such challenges a number of methods
have been introduced to speed up training of continuous 
time models including
regularizing dynamics~\cite{kelly_learning_2020, finlay_how_2020} 
and replacing ordinary differential equation (ODE)
solvers with integration methods where possible~\cite{kidger_hey_2021}.
Despite the computational advancements brought about 
by such innovations, continuous time models 
remain expensive to train in comparison 
to discrete time models.

In addition to these computational challenges,
it is well-known that adjoint methods suffer 
from stability issues when approximating gradients of time 
averaged quantities over long time intervals 
for chaotic systems~\cite{lea_sensitivity_2000}. 

In the current work, we present a memory and time efficient 
method for inferring nonlinear, latent stochastic differential 
equations (SDEs) from high-dimensional time-series datasets.
In contrast to standard approaches for inferring latent
differential equations that rely on
adjoint sensitivities~\cite{chen_neural_2018,rubanova_latent_2019,li_scalable_2020},
our approach removes the requirement of solving
differential equations entirely.
We accomplish this advancement by coupling a novel 
amortization strategy with a recently derived 
reparametrization for expectations under 
Markov Gaussian processes~\cite{course_state_2023}.
We show that our approach can be used to approximate 
gradients of the evidence lower bound used to train 
latent SDEs with a 
time and memory cost that is 
independent of the amount of data, 
the length of the time series, and 
the stiffness of the approximate differential equations.
The asymptotic complexity for our approach is compared  
to well-known methods from the literature in 
Table~\ref{tab:time-and-memory-comparison}.
We note that our method has a constant cost 
that is chosen by the user. 
Moreover, we will show that our method is embarrassingly parallel
(i.e. all evaluations of the model can be performed in parallel over each 
iteration) whereas, we reiterate, 
differential equation solvers are iterative sequential methods.

\begin{table}[htbp]
  \begin{center}
  \begin{tabular}{lccc} \\ \toprule
    Method & Time & Memory\\ 
    \hline 
    Deterministic adjoints (Neural ODE)~\cite{chen_neural_2018} & $\mathcal{O}(J)$ & $\mathcal{O}(1)$ \\
    Stochastic adjoints~\cite{li_scalable_2020} & $\mathcal{O}(J\log J)$ & $\mathcal{O}(1)$\\
    Backprop through solver~\cite{giles_smoking_2006} & $\mathcal{O}(J)$ & $\mathcal{O}(J)$\\
    Amortized reparametrization (ours) & $\mathcal{O}(R)$ & $\mathcal{O}(R)$
          \\ \bottomrule
  \end{tabular}
\caption{\label{tab:time-and-memory-comparison}
Asymptotic complexity comparison for approximating gradients.
Units are given by the number of evaluations of the differential equation 
    (gradient field for ODEs or drift and diffusion function for SDEs).
    Here $J$ is the number
  of {\em sequential} evaluations 
  and $R$ is the number of {\em parallel} evaluations.
  $J$ is adaptively chosen by the differential equation solver and is 
    a function of the stiffness of the differential equation meaning it can change (and possibly explode)
    while optimizing. 
  In contrast, $R$ is a fixed constant used to control the variance of gradient approximations.
    While we could choose $R=1$ and would still arrive with unbiased approximations 
    for gradients, in 
  practice we found choosing $R\approx10^2$ worked well for the problems we considered.
  }
  \end{center}
\end{table}

The applications of our method span various standard generative modeling tasks, such as 
auto-encoding, denoising, inpainting, and super-resolution~\cite{kingma_auto-encoding_2014},
particularly tailored for high-dimensional time-series.
Crucially, the computational 
efficiency of our approach not only enables the allocation  
of more computational resources towards
hyperparameter tuning but also 
democratizes access to state-of-the-art methods by making 
them feasible to train on lower 
performance hardware.

In the next section we provide a description of the 
theory underpinning our work with the main result of a stochastic,
unbiased estimate for gradients appearing in Lemma~\ref{lemma:unbiased-gradient}.
In Section~\ref{sec:numerical-studies} we provide a number 
of numerical studies including learning latent neural SDEs from video
and performance benchmarking on a motion capture dataset.
Notably we show that
we are able to achieve comparable
performance to methods based on adjoints with more than {\bf one order
of magnitude} fewer evaluations of the model in training (Section~\ref{sec:lotka-volterra}).
We also demonstrate that our approach does not suffer 
from the numerical instabilities that plague adjoint methods for long 
time-series with chaotic systems (Section~\ref{sec:adjoints}).
Finally, we close with a discussion of the limitations 
of our approach as well as some suggestions for 
future work.
All code is available at
\href{https://github.com/coursekevin/arlatentsde}{github.com/coursekevin/arlatentsde}.

\section{Method}

\subsection{Problem description}
Consider a time-series dataset
$
  \mathcal{D} = \left\{(x_i, t_i)\right\}_{i=1}^N
$,
where $t_i\in\R$ is the time stamp associated with the
observation, $x_i\in\R^D$. For example $x_i$ may 
be an indirect observation of some underlying dynamical
system,
a video-frame, or a snapshot of a spatio-temporal field.
We assume that each observation was generated via 
the following process:
first a latent trajectory, $z(t)$, is sampled from a general, \emph{nonlinear} SDE 
with time-dependent diffusion,
\begin{equation}
  d{z} = f_\theta(t, z) dt + L_\theta(t) d\beta,
    \label{eq:generative-model}
\end{equation}
with initial condition $\p_\theta(z_0)$
and then each $x_i$ is generated from the 
conditional distribution $p_\theta(x \mid z(t_i))$ 
where $z(t_i)$ is the 
realization of the latent trajectory at time stamp $t_i$.
Here $f_\theta:\R\times \R^d \to \R^d$ is called the drift function,
$L_\theta:\R \to \R^{d\times d}$ is called the dispersion matrix, 
and $\beta $ denotes Brownian motion with diffusion 
matrix $\Sigma$. 
We will model both the drift function of the SDE and the conditional
distribution using neural networks in this work.
This combination of SDE and conditional distribution define the generative model 
that we wish to infer.

Like in the standard generative modeling setting, the ``true'' 
parameters defining the latent SDE and the conditional likelihood, 
as well as the specific realization of the latent state, 
remain unknown. 
In addition, the posterior over the latent state at a 
point in time, $p_\theta(z(t)\mid \mathcal{D})$, is intractable.
Our objectives are two-fold, we wish to:
(i)~infer a likely 
setting for the parameters, $\theta$, 
that is well-aligned with the observed data
and (ii)~infer a parametric model for the posterior
over the latent state at a particular point in time
as a function of a small window of the observations, for example
$q_\phi(z(t) \mid x_i, x_{i+1}, \dots, x_{i+M}) \approx
p_\theta(z(t) \mid \mathcal{D})$ for $t\in[t_i, t_{i+M}]$ and 
$1 \leq M << N$. This latent variable model is depicted 
in Figure~\ref{fig:generative-model}.
In this work we will tackle this problem statement 
using the machinery of stochastic variational inference~\cite{hoffman_stochastic_2013}.

\begin{figure}[htbp]
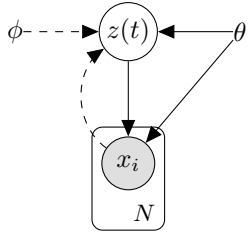

  \begin{center}
\tikz{
    \node[obs] (x) {$x_i$};%
    \node[latent,above=of x] (z) {${z}(t)$}; %
    \node[const, right=of z] (theta) {$\theta$};
    \node[const, left=of z] (phi) {$\phi$};
     \plate {plate1} {(x)} {$N$}; %
     \edge {z} {x};
     \edge{theta} {z, x};
      \edge[dashed] {phi} {z};
      \draw [->,dashed] (x) to [out=145,in=215] (z)}
  \caption{The solid lines indicate the data generating process
    that depends on the parameters, $\theta$. The dashed lines indicate
    the approximate variational posterior that depends on the parameters, $\phi$.}
  \label{fig:generative-model}
  \end{center}
\end{figure}

Before proceeding with an explanation of our approach, 
it is worthwhile taking a moment to consider 
why assuming the latent state is a realization of a 
SDE, as opposed to an ordinary differential equation (ODE) 
with random initial conditions~\cite{chen_neural_2018,rubanova_latent_2019,toth_hamiltonian_2020},
merits the additional mathematical machinery.
First, there is a long history in using SDEs in 
the physical sciences 
to model systems (even deterministic systems) that
lack predictability due to super-sensitivity to parameters 
and or initial conditions~\cite{glimm_stochastic_1997}.
One reason for the preference towards this 
modeling paradigm is that deterministic ODEs with random initial conditions only permit
uncertainty to enter in their initial condition.
This means that we are assuming that all randomness
which may affect the latent trajectory at 
all future points in time are precisely 
encoded in the uncertainty in the initial condition. 
In contrast, SDEs make the less restrictive 
assumption that uncertainty accumulates over time.

\subsection{Evidence lower bound}
As is the case with standard latent variable models,
training proceeds by specifying approximate posteriors over latent variables 
(in this case the latent state, $z(t)$) and then minimizing 
the Kullback-Leibler (KL) divergence between the approximate and true posterior.
An added complication in our specific case comes from 
the fact that the latent state is defined by a stochastic 
process rather than a random-vector as is more typical.
While it is possible to choose more general approximate posteriors~\cite{li_scalable_2020},
in this work we make the choice to approximate the posterior over
the latent state as a Markov Gaussian Process (MGP) defined by the linear 
SDE,
\begin{equation}
  dz = (-A_\phi(t)z + b_\phi(t))dt + L_\theta(t)d\beta,
\end{equation}
with Gaussian initial condition, $z(t_0)\sim \mathcal{N}(m_0, S_0)$ where 
$m_0\in\R^{d}$ and $S_0\in\R^{d \times d}$ is a symmetric positive definite. Here
$A_\phi:\R\to\R^{d\times d}$ and $b_\phi:\R\to\R^d$ are symmetric matrix 
and vector-valued functions of time, respectively.
Before proceeding, it is worth emphasizing that this choice of approximate 
posterior is not as restrictive as it may appear since it is 
only used to approximate the \emph{posterior} (i.e. the latent
state given the observations $p(z(t) \mid \mathcal{D})$ for $t\in[t_1, t_N]$). When forecasting
beyond the data window, we will use the nonlinear SDE in~\eqref{eq:generative-model}.
In this way, the approximate posterior primarily serves as a useful intermediary
in training the generative model rather than the end-product.

For our purposes, MGPs are useful because the 
marginal distribution $q_\phi(z(t)) = \mathcal{N}(m_\phi(t), S_\phi(t))$ (i.e. the distribution 
of the latent state at time $t$)
is a Gaussian whose mean and covariance 
are given by the solution to the following 
set of ODEs~\cite{sarkka_applied_2019},
\begin{equation}
  \begin{split}
    \dot{m}_\phi(t) &= -A_\phi(t) m_\phi(t) + b_\phi(t),\\
    \dot{S}_\phi(t) &= -A_\phi(t) S_\phi(t) - S_\phi(t) A_\phi(t)^T + L_{\theta}(t)\Sigma L_{\theta}(t)^T,\\
  \end{split}
\end{equation}
where $m(0)=m_0$ and $S(0) = S_0$.

Using the reparametrization trick from~\cite{course_state_2023},
the evidence lower bound (ELBO) can be written as,
\begin{equation}
  \begin{split}
    \mathcal{L}(\theta, \phi) =& 
    \sum_{i=1}^N \E{\log \p_\theta(x_i \mid z(t_i))}{z(t_i)\sim q_\phi(z(t))} \\
    &-\frac{1}{2} \int_0^T \E{\rnorm{r_{\theta,\phi}(z(t),t)} }{z(t)\sim q_\phi(z(t))}\,dt, 
  \end{split}
  \label{eq:elbo-svise}
\end{equation}
where
\begin{equation}
  \begin{split}
    r_{\theta,\phi}(z(t),t) =& B(t)(m_\phi(t)-z(t))
    + \dot{m}_\phi(t)- f_\theta(z(t), t),
  \end{split}
\end{equation}
$C_{\theta}(t) = (L_\theta(t) \Sigma L_\theta(t)^{T})^{-1}$,
$B(t) = \text{vec}^{-1}\left((S_\phi(t) \oplus S_\phi(t))^{-1}
    \text{vec}(C_\theta(t)^{-1} - \dot{S}_\phi(t))\right)$,
$\oplus$ indicates the Kronecker sum,
$\text{vec}:\R^{d\times d}\rightarrow \R^{d^2}$ maps
a matrix into a vector by stacking columns, and $\text{vec}^{-1}: \R^{d^2}\rightarrow\R^{d\times d}$
converts a vector into a matrix such that 
$\text{vec}^{-1}(\text{vec}(B)) = B$ $\forall B\in\R^{d\times d}$. 
See Appendices~\ref{app:identity} and \ref{app:elbo} for a detailed derivation. 
We note that computing the residual, $r_{\theta,\phi}$, scales 
linearly in the dimension of the state so long as $S_\phi(t)$ and $C_\theta(t)$ 
are diagonal.
To ensure our approach remains scalable, we will make this 
assumption throughout the remainder of this work.
In this case, we have
\begin{equation}
    B(t) = \frac{1}{2}S_\phi(t)^{-1} ( C_\theta(t)^{-1} - \dot{S}_\phi(t)).
\end{equation}

Often we may wish to place additional priors onto a subset of the generative
model parameters, $\theta$,
and infer their posterior using stochastic variational inference as well.
In this case we add the KL-divergence between the approximate posterior and prior 
onto the ELBO, ${KL}(q_\phi(\theta) \mid\mid p(\theta))$ and rewrite all 
expectations with respect to $q_\phi(z(t))$ and $q_\phi(\theta)$; details 
are provided in Appendix~\ref{app:priors-on-theta}.

\paragraph{Remark 1:} Despite having defined the approximate posterior in 
terms of an SDE with parameters $A_\phi$ and $b_\phi$, the ELBO 
only depends on the mean and covariance of the process at a particular 
point in time, $m_\phi$ and $S_\phi$. 
For this reason we can parametrize $m_\phi$ and $S_\phi$ directly while
implicitly optimizing with respect to $A_\phi$ and $b_\phi$. 
In addition, we can efficiently compute $\dot{m}_\phi$ and $\dot{S}_\phi$ using 
automatic differentiation. 
\looseness=-1

\paragraph{Remark 2:} Despite the fact that the prior and approximate 
posterior are SDEs, all expectations in the ELBO
are taken with respect to
{\em normal distributions}. 
Moreover, in contrast to the approach
in~\cite{archambeau_gaussian_2007,archambeau_variational_2007}
there are no differential equality constraints -- instead we have been left with an integral 
over the window of observations. 

Taken together, these observations allow us to infer the parameters of the 
generative model (a nonlinear, latent SDE with additive
diffusion~\eqref{eq:generative-model}), without the use of a forward solver.

\subsection{Amortization strategy}
The implicit assumption in the ELBO in~\eqref{eq:elbo-svise} is 
that the state mean and covariance will be approximated over the {\em entire}
window of observations.
This can pose a serious computational bottleneck  with 
long or complicated time-series. 
In this section we propose a novel amortization strategy
that will allow us to effectively eliminate this cost 
by requiring that we only approximate the posterior over 
short partitions of total the data time-window at once.

Rather than attempting to compute and store the posterior over the 
entire latent trajectory, we will instead construct an 
approximation to the posterior over a small window of observations 
as a function of those observations.
Consider a reindexing of the dataset by splitting it into $N/M$ non-overlapping 
partitions where $1 \leq M <<N$,
\begin{alignat*}{7}
  \text{original indexing: }\quad&[t_1,       &&t_2,       &&\dots,       &&t_M, &&t_{M+1}, \dots, &&t_N &&] \\
  \text{reindexed dataset: }\quad&[t_1^{(1)}, &&t_2^{(1)}, &&\dots, &&t_M^{(1)}, &&t_{1}^{(2)}, \dots, &&t_M^{(N/M)} &&] 
\end{alignat*}

In the case that $N$ is not evenly divisible by $M$ we allow the final split to contain
less elements.
We approximate the latent state over each partition using only the $M$ observations
in each partition,
$q_\phi(z(t)\mid x_1^{(j)}, x_{2}^{(j)}, \dots, x_{M}^{(j)}) \approx p(z(t)\mid \mathcal{D })$ 
for $t\in [t_1^{(j)}, t_{1}^{(j+1)}]$.
This can be interpreted as a probabilistic encoder over the time interval of the
partition of observations.
Letting $t_1^{N/M + 1} \equiv t_M^{N/M}$, 
the ELBO can be compactly rewritten as, 
$
  \mathcal{L}(\theta,\phi) = \sum_{j=1}^{N/M}\mathcal{L}^{(j)}(\theta, \phi),
  \label{eq:amortized-elbo}
$
where
\begin{equation}
  \begin{split}
    \mathcal{L}^{(j)}(\theta, \phi) =&\sum_{i=1}^M \E{\log p_\theta(x_i^{(j)}\mid z(t_i^{(j)}))}
                                                              {q_\phi(z(t_i^{(j)})\mid x_1^{(j)}, \dots, x_M^{(j)})}\\
    &-\frac{1}{2}\int_{t_1^{(j)}}^{t_1^{(j+1)}} \E{\rnorm{r_{\theta,\phi}(z,t)}}
    {q_\phi(z(t)\mid x_{1}^{(j)}, \dots, x_{M}^{(j)})}\,dt.
  \end{split}
    \label{eq:amortized-objective}
\end{equation}

An additional advantage of this amortization strategy is that it allows our
approach to scale to multiple trajectories without an increase to the overall
computational cost.
If there are multiple trajectories, we can reindex each trajectory 
independently and subsequently sum all sub loss functions.

To reiterate, the probabilistic encoder is a function which takes in $M$
observations from a particular partition along with a time stamp, $t$,
and outputs a mean vector and covariance matrix as an estimate for the latent state
at that particular time.
In principle, any function which can transform a batch of snapshots and a time
stamp into a mean and covariance could be used as an encoder in our work.
In our implementation, we use deep neural networks to encode each $x_i^{(j)}$
using $i\in\mathcal{I}$ where $\mathcal{I}\subset [x_1^{(j)}, x_2^{(j)}, \dots,
x_M^{(j)}]$ contains some temporal neighbours
of $x_i$ into a single latent vector.
This approach yields a set of latent vectors associated with each observation
in the partition $h_i$ for $i=1,2,\dots, M$.
We then interpolate between each latent vector using a deep kernel based
architecture to construct the posterior approximation for any time stamp
in the partition; see Appendix~\ref{app:encoder} for details.
We emphasize this is one choice of encoding architecture that we found 
convenient, it is straightforward to incorporate an autoregressive 
or transformer based encoder in our methodology~\cite{vaswani_attention_2017}

An important consideration is the selection of the partition parameter, $M$.
In practice, $M$ should be large enough so that the assumption of a linear SDE 
for the approximate posterior is appropriate (i.e. we have enough observations in a
partition so that the assumption of a Gaussian process over the latent state is
reasonable). For example, as we will see in an upcoming numerical study,
in the context of inferring latent SDEs from videos, we will need to choose
$M$ to be large enough so that we can reasonably infer both the position and
velocity of the object in the video.

\subsection{Reparametrization trick}\label{sec:reparam-trick}
While the previous sections have demonstrated 
how to eliminate the need for a differential equation
solver by replacing the initial value problem with an integral, 
in this section we show how the reparametrization trick
can be combined with the previously described amortization 
strategy to construct unbiased gradient approximations for the ELBO
with a time and memory cost that scales independently with the 
amount of data, the length of the time series, and the stiffness 
of the approximation to the differential equations.
Consider a reparametrization to the 
latent state of the form $z(t) = T(t, \epsilon, \phi)$ where 
$\epsilon \sim p(\epsilon)$ so that $z(t)\sim q_\phi(z(t)\mid x_1^{(j)}, 
x_{2}^{(j)}, \dots, x_{M}^{(j)})$. 
We can rewrite the second term in the evidence lower bound as,
\begin{equation}
  \begin{split}
  \int_{t_1^{(j)}}^{t_1^{(j+1)}} \E{\rnorm{r_{\theta,\phi}(z(t),t)}}{q_\phi(z(t))}\,dt &= 
    \int_{t_1^{(j)}}^{t_1^{(j+1)}} \E{\rnorm{r_{\theta,\phi}(T(t,\epsilon, \phi),t)}}{p(\epsilon)}\,dt\\
    &= ( t_1^{(j+1)} - t_1^{(j)} )\E{\rnorm{r_{\theta,\phi}(T(t,\epsilon, \phi),t)}}{p(\epsilon)p(t)}  
  \end{split}
\end{equation}
where $p(t)$ is a uniform distribution, $\mathcal{U}(t_1^{(j)}, t_1^{(j+1)})$ 
and $p(\epsilon)\sim\mathcal{N}(0,I)$ is a Gaussian.
With this rearrangement, we can derive the main result of this
work.
\begin{lemma} \label{lemma:unbiased-gradient}
  An unbiased approximation of the gradient of the evidence lower bound, 
  denoted as $\nabla_{\theta,\phi}\mathcal{L}(\theta,\phi)$, with an 
  $\mathcal{O}(R)$ time and memory cost can be formulated as follows:
\begin{equation}
  \begin{split}
    \nabla_{\theta,\phi}\mathcal{L}(\theta, \phi) \approx& \frac{N}{R}\sum_{i=1}^M \sum_{k=1}^{R} 
            \nabla_{\theta,\phi}\log p_\theta(x_i^{(j)}\mid T(t_i^{(j)}, \epsilon^{(k)}, \phi))\\
            &- ( t_1^{(j+1)} - t_1^{(j)} )\frac{N}{2R}\sum_{k=1}^R
                    \nabla_{\theta,\phi}  \lvert\lvert r_{\theta,\phi}(T(t^{(k)},\epsilon^{(k)}, \phi),t^{(k)})\rvert\rvert_{C_\theta(t^{(k)} ) }^2.
  \end{split}
\end{equation}
where each $t^{(k)} \sim \mathcal{U}(t_1^{(j)}, t_1^{(j+1)})$ 
and each $\epsilon^{(k)}\sim \mathcal{N}(0, I)$.
\end{lemma}
The proof follows by applying the standard reparametrization 
trick~\cite{kingma_auto-encoding_2014} to estimating gradients of the 
amortized objective in~\eqref{eq:amortized-objective}.

\paragraph{Remark 1: } In practice we found 
choosing $R\sim 100$ worked well for the problems we 
considered.
Note that in terms of elapsed time, $100$ 
evaluations of this objective, which can be computed 
in parallel, is far cheaper than $100$ evaluations
of the SDE forward model evaluated as a part 
of an iterative sequential SDE solver. 
Moreover we found that adaptive stepping schemes
required far more evaluations of the SDE forward 
model than our stochastic approach (see Section~\ref{sec:lotka-volterra}).

\paragraph{Remark 2: }
In the case that 
evaluations of the SDE drift term were relatively 
cheap compared to decoder evaluations (for example in the case the dimension
of the latent state is much smaller than the dimension of the data), 
we found it useful to increase the number of samples used to approximate the 
integral over time without increasing the number of samples from the 
variational posterior.
To do so, we made use of a nested Monte Carlo scheme
to approximate the second term in the ELBO,
\begin{multline}
( t_1^{(j+1)} - t_1^{(j)} )\E{\rnorm{r_{\theta,\phi}(T(t,\epsilon, \phi),t)}}{p(\epsilon)p(t)}  \approx\\
            \frac{ t_1^{(j+1)} - t_1^{(j)} }{RS}\sum_{k=1}^R \sum_{l=1}^S 
                      \lvert\lvert r_{\theta,\phi}(T(t^{(k, l)},\epsilon^{(k)}, \phi),t^{(k, l)})\rvert\rvert_{C_\theta(t^{(k, l)})}^2,
          \label{eq:nested-monte-carlo}
\end{multline}
where, again, each $\epsilon^{(k)} \sim \mathcal{N}(0, I)$ and each 
$t^{(k,1)}, t^{(k, 2)}, \dots, t^{(k, S)}\sim \mathcal{U}(t_1^{(j)}, t_1^{(j+1)})$.
In addition, because the integral over time is one-dimensional we used 
stratified sampling to draw from $\mathcal{U}(t_1^{(j)}, t_1^{(j+1)})$
in order to further reduce the variance in the integral over time.
In this case we often 
found we could choose $R \sim 10$ and $S\sim 10$.
To be clear, \eqref{eq:nested-monte-carlo} is simply a method for variance 
reduction that we found to be useful; it is not a necessary component for our
approach.

\section{Limitations \& Related Work}

\paragraph{Summary of assumptions.}
In the previous sections we introduced an ELBO which, when maximized,
leaves us with a generative model in the form of a nonlinear, latent SDE 
with time-dependent diffusion and an approximation to the latent 
state over the time-window of observations in the form of a Gaussian process.
To reiterate, we only assume that the approximating posterior, i.e. the 
distribution over the latent state given a batch of observations, is a Gaussian 
process; this is an assumption that is commonly made in the context 
of nonlinear state estimation, for example~\cite{barfoot_state_2017,barfoot_exactly_2020}.
When making predictions, we sample from the nonlinear SDE which characterizes 
the generative model~\eqref{eq:generative-model}.

\paragraph{Stochastic adjoint sensitivities.}
Li et al.~\cite{li_scalable_2020} proposed the stochastic adjoint 
sensitivity method, 
enabling the inference of latent SDEs using a wide range of 
approximate posteriors over the latent state. 
In our work we choose to approximate the 
posterior over the latent state using a MGP which enables us to eliminate 
the requirement of solving any differential equations entirely; 
as we have discussed extensively this choice enables dramatic 
computational savings.
A limitation of our approach as compared to the stochastic 
adjoint sensitivities method is that our method should
only be used to approximate the posterior over the latent 
state when it is approximately a MGP.
Intuitively, this limitation is akin to the limitations 
of mean-field stochastic variational inference as 
compared to stochastic variational inference with an 
expressive approximate posterior such as 
normalizing flows~\cite{rezende_variational_2015}.
From our practical experience working on a range of 
test cases, this has not been a limiting factor.
It is worth reiterating that this limitation 
applies only to the approximate posterior over the time window 
of observations; the predictive posterior can be a complex 
distribution defined by a nonlinear SDE with a Gausian 
initial condition.

In addition, the stochastic adjoint sensitivity method 
allows for state dependent diffusion processes whereas 
our approach only allows for a time dependent diffusion process. 
In cases where a state dependent diffusion process is deemed 
necessary, our approach could be used to provide 
a good initial guess for the parameters of the drift function.
It remains a topic of future 
work to determine if this limitation is mitigated by the fact 
that we are learning latent SDEs rather than SDEs in the 
original data space.
Across the range of 
test cases we considered, we have not encountered 
a problem for which the assumption of a time-dependent 
diffusion matrix was limiting.

\paragraph{Latent neural ODEs.}
Chen et al.~\cite{chen_neural_2018},
Rubanova et al.~\cite{rubanova_latent_2019}, and Toth et al.~\cite{toth_hamiltonian_2020} 
presented latent ordinary differential equations (ODEs) as generative 
models for high-dimensional temporal data. 
These approaches have two main limitations: 
(i) they encode all uncertainty in the ODE's 
initial condition and
(ii) they rely on adjoint sensitivities, 
necessitating the solution of a sequence of 
initial value problems during optimization. 
As was discussed previously, SDEs provide a more natural 
modeling paradigm for estimating uncertainty,
naturally capturing our intuition that 
uncertainty should accumulate over time~\cite{glimm_stochastic_1997}. 
Moreover, to reiterate, 
our work avoids solving differential equations entirely
by relying on unbiased approximations of a one-dimensional 
integral instead; as we will show, this can 
result in a dramatic decrease in the number of 
required function evaluations in training as compared to 
methods based on adjoints.
Moreover, we will show that our approach 
avoids the numerical instabilities of adjoint 
methods when they are used to approximate gradients
of time averaged quantities over long time intervals 
for chaotic systems.
It is worth mentioning that gradients computed by backpropagation 
of a forward solver are not consistent with the adjoint ODE in
general~\cite{alexe_discrete_2009} so we do not consider 
comparisons to such approaches here.

\paragraph{Weak form methods.} 
Methods for inferring continuous time models 
of dynamical systems using the weak form of the 
differential equations were introduced in the 
context of learning ODEs with linear dependence
on the parameters~\cite{schaeffer_sparse_2017,pantazis_unified_2019}.
More recently these methods were adapted for training 
neural ODEs more quickly than adjoint methods
for time-series prediction problems~\cite{course_weak_2020}.
These methods share some similarities to 
the present approach in how they achieve a computational 
speed-up -- both methods transform the problem 
of solving differential equations into a problem 
of integration.
In contrast to the present approach,
these methods only allow for one to learn 
an ODE in the data coordinates (i.e. they do not allow for 
one to infer an autoencoder and a set of differential equations
simultaneously). 
Moreover, these methods rely on a biased estimate 
for the weak form residual which will fail when 
observations become too widely spaced.
In contrast, in the present approach, we rely on unbiased 
approximations to the evidence lower bound.
Finally, these methods require the specification 
of a carefully designed test-space~\cite{messenger_weak_2021} --
a consideration not required by our approach.

\section{Numerical Studies}\label{sec:numerical-studies}
In this section we provide a number of numerical studies to demonstrate
the utility of our approach.
In the first study, we show that our approach can be used to train neural SDEs 
using far fewer evaluations of the model than adjoint methods.
In the second study, we consider the problem of parameter tuning for a chaotic system 
over long time intervals. We show that our approach does not suffer from the numerical 
instabilities which are known to cause issues with adjoint methods on problems such as these.
Finally we close this section with two practical test cases: the first demonstrating 
competitive performance on a motion capture benchmark and the second showing 
how our approach can be applied to learn neural SDEs from video.
An additional numerical study exploring the effect of the nested Monte Carlo 
parameter, $S$, is provided in Appendix~\ref{app:study-on-mc}.
Details on computing resources are provided
Appendix~\ref{app:computing-resources}. 
All code required to reproduce results and figures is provided at
\href{https://github.com/coursekevin/arlatentsde}{github.com/coursekevin/arlatentsde}.

\subsection{Orders of magnitude fewer function evaluations in training}\label{sec:lotka-volterra}
\begin{figure}[htb]
  \centering
\begin{subfigure}{.5\textwidth}
  \centering
  \includegraphics[width=1.\linewidth]{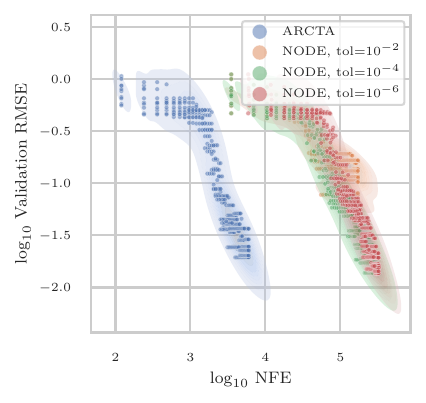}
  \label{fig:sub1}
\end{subfigure}%
\begin{subfigure}{.5\textwidth}
  \centering
  \includegraphics[width=1.\linewidth]{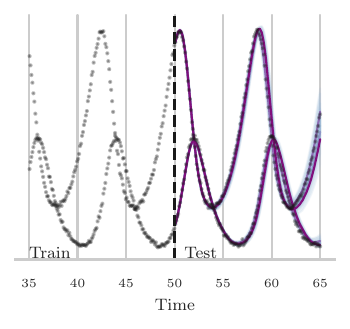}
  \label{fig:sub2}
\end{subfigure}
  \caption{Lotka-Volterra benchmarking result. In the left figure we see our method (\textcolor{mplblue}{ARCTA}) 
  requires more than 
  {\bf one order of magnitude} fewer evaluations of the model (NFE) than the standard 
  neural ODE
  (\textcolor{mplgreen}{NODE}) to achieve a similar validation accuracy. 
  In the right figure
  we have plotted a probabilistic prediction on the test set along with 
  three samples from the predictive distribution.}
\label{fig:NFE-comparison}
\end{figure}
In this numerical study we consider the task of building 
a predictive model from noisy observations of a predator-prey 
system. 
We simulated the Lotka-Volterra equations for $50$ seconds 
collecting data at a frequency of $10$Hz. 
Observations were 
corrupted by Gaussian noise with a standard deviation of $0.01$.
Validation data was collected 
over the time inverval $[50, 65]$ seconds.
We then attempt to build a predictive model from the data using 
a neural ODE (NODE) and our method,
amortized reparametrization for continuous time auto-encoding
(ARCTA), with the same model 
for the ODE and drift function respectively. 
To make comparisons with the NODE fair, we set the 
decoder to be the identity function.
We assume the diffusion matrix is constant 
and place a log-normal prior on its diagonal elements.
We approximate the posterior over these elements 
using a log-normal variational posterior. 
Details on the architecture and hyperparameters 
are provided in Appendix~\ref{app:arch-lotka-volterra}.
For this experiment, as well as subsequent experiments, 
we made use of the Adam optimizer~\cite{kingma_adam_2015}.

We considered three different tolerances on the NODE adaptive 
stepping scheme.
We trained our model as well as the NODEs using 10 different random seeds while
recording the validation RMSE and the number of evaluations of the model.
Looking to Figure~\ref{fig:NFE-comparison}, we see that our 
approach required more than an order of magnitude 
fewer evaluations of the model to achieve a similar RMSE 
on the validation set. 
This remains true even when the tolerance of the ODE solver 
is reduced such that the validation RMSE is substantially 
higher than our approach.

\subsection{Numerical instabilities of adjoints} \label{sec:adjoints}
It is well-known that adjoint based methods produce 
prohibitively large gradients for long time averaged quantities 
of chaotic systems~\cite{lea_sensitivity_2000} and accordingly
methods, such as least squares shadowing~\cite{wang_least_2014},
have been introduced to address such concerns. 
In this section we reproduce this phenomena on a simple parameter 
tuning problem and show that our approach does not 
suffer these same issues.

Given the parametric form of the chaotic Lorenz equations,
\begin{align}
  \dot{x} &= \sigma (y - x)\label{eq:dx-lorenz} \\
  \dot{y} &= x(\rho - z) - y \label{eq:dy-lorenz}\\ 
  \dot{z} &= x y - \beta z \label{eq:dz-lorenz}
\end{align}
along with an initial guess for the parameters, $\sigma_0$, 
$\rho_0$, and $\beta_0$, our goal is to tune the 
value of parameters such that they align with the observed data. 

\begin{wrapfigure}[23]{o}{0.50\textwidth}  
  \vspace{-20pt}
  \centering
  \includegraphics[width=1.0\linewidth]{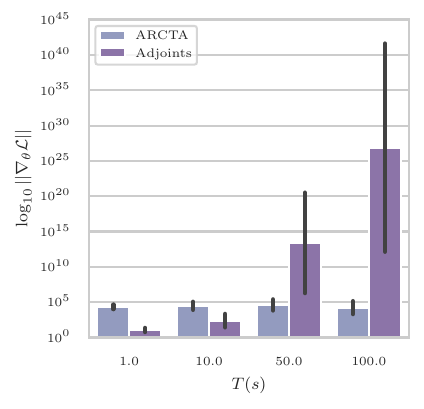}
  \caption{Stability of gradients in chaotic systems. The log-scale on vertical axis shows 
    our approach remains stable for longer time series, while adjoint-based 
    gradients become unusable at $50$ and $100$ seconds.}
  \label{fig:adjoints-instability}
\end{wrapfigure}
For this experiment we collect data at a frequency of 200Hz 
and corrupt observations by Gaussian noise with a covariance of $1$.
We generate five independent datasets over the time intervals
$[0, 1]$, $[0, 10]$, $[0, 50]$, and $[0, 100]$. For each dataset 
we generated an initial guess for the parameters by sampling 
from a Gaussian whose mean is the true value of the parameters 
and standard deviation is $20$\% of the mean.
For the adjoint methods we report the $\ell_2$-norm of the gradient 
with respect to the parameters at the initial guess. 
For our method (ARCTA) we optimize for $2000$ iterations (which tended
to be enough iterations to successfully converge to a reasonable 
solution) and report the average gradient across all iterations.
Details on hyperparameters and our architecture design are 
provided in Appendix~\ref{app:arch-adjoints}.
Results are summarized in Figure~\ref{fig:adjoints-instability}.
While adjoints expectedly provide prohibitively large gradients
as the length of the time series is increased, our approach 
remains numerically stable.

\subsection{Motion capture benchmark}\label{sec:mocap}
\begin{figure}[htb]
  \begin{center}
    \begin{minipage}{0.49\linewidth}
        \centering \begin{tabular}{c c}\\ \toprule
          Method & Test RMSE \\
          \hline
              DTSBN-S~\cite{gan_deep_2015} & $5.90 \pm 0.002^\dag$ \\
              npODE~\cite{heinonen_learning_2018} & $4.79^\dag$ \\
              NeuralODE~\cite{chen_neural_2018} & $4.74 \pm 0.093^\dag$ \\
              ODE$^2$VAE~\cite{yildiz_ode2vae_2019} & $3.17 \pm 0.221^\dag$ \\
              ODE$^2$VAE-KL~\cite{yildiz_ode2vae_2019} & $2.84 \pm 0.343^\dag$ \\
              Latent ODE~\cite{rubanova_latent_2019}& $2.45 \pm 0.057^\ast$ \\
              Latent SDE~\cite{li_scalable_2020} & $2.01 \pm 0.050^\ast$ \\
              ARCTA (ours) & $2.76 \pm 0.168$\\
           \bottomrule
        \end{tabular}
    \end{minipage} %
    \begin{minipage}{0.49\linewidth}
      \centering
        \centering
        \includegraphics[width=\linewidth]{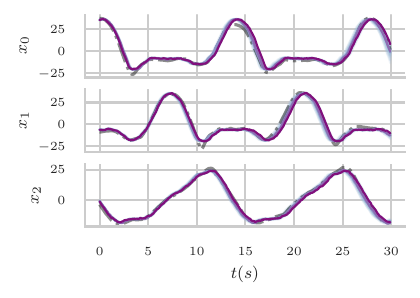}
    \end{minipage}
    \caption{MOCAP benchmarking results, $^\dag$ from~\cite{yildiz_ode2vae_2019} 
    and $^\ast$ from~\cite{li_scalable_2020}.
    Our score is computed by training 10 models with different seeds and
      averaging on the test set.
    Looking to the table, we see
    that our method performs similarly to other state-of-the-art methods.
    The plot shows the predictive posterior on the test set for some select 
    outputs. Other benchmark results were compiled in~\cite{yildiz_ode2vae_2019, li_scalable_2020}.
    RMSE was computed from MSE by taking the square root of the mean and transforming 
    the error via a first-order Taylor-series approximation.
    }
    \label{fig:mocap-results}
  \end{center}
\end{figure}
In this experiment we consider the motion capture dataset
from~\cite{gan_deep_2015}.
The dataset consists of 16 training, 3 validation, 
and 4 independent test sequences of a subject walking.
Each sequence consists of $300$ time-series observations 
with a dimension of $50$.
We made use of the preprocessed data from~\cite{yildiz_ode2vae_2019}.
Like previous approaches tackling this dataset, 
we chose a latent dimension of $6$. 
We assume a Gaussian observation likelihood.
We place a log-normal prior on the diagonal elements 
of the diffusion matrix and the noise on the observations.
We approximate the posterior of the diffusion matrix 
and observation noise covariance using a log-normal 
approximate posterior.
Details on our architecture design and hyperparameter 
selection are provided in Appendix~\ref{app:arch-mocap}.

For our approach, we train 10 models and report their average performance 
on the test set due to the extremely limited number (4) of independent 
test sequences.
Looking to Figure~\ref{fig:mocap-results}, we see that 
our approach provided competitive performance on this 
challenging dataset. This result, in combination with 
those presented previously demonstrating we require fewer
function evaluations for similar forecasting accuracy 
and improved gradient stability for chaotic systems, 
make clear the utility of the present work.
It is possible to achieve state-of-the-art performance 
at a significantly reduced computational cost 
as compared to adjoint based methods.

\subsection{Neural SDEs from video}\label{sec:nsde-from-video}
\begin{figure}[htb]
  \centering
  \includegraphics[width=1.0\textwidth]{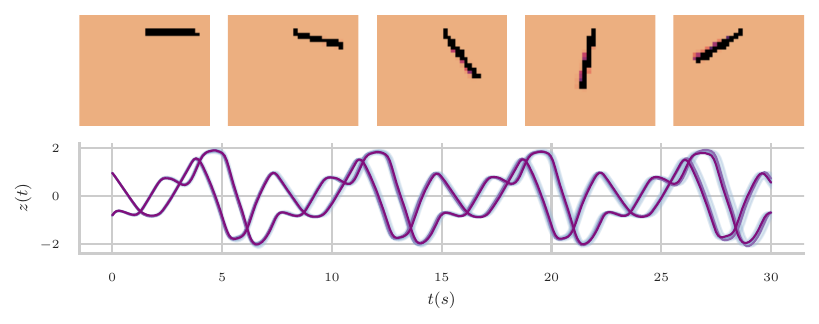}
  \caption{Neural SDEs from video. Here we used five frames to 
    estimate the intial state and then forecast in the latent 
    space for 30 seconds. The bottom plot shows 
  the latent SDE. The top row shows 10 samples 
  from the predictive posterior overlaid on the data.}
  \label{fig:neural-sdes-from-video}
\end{figure}
In this experiment we attempt to learn a latent SDE from video. 
We generated $32\times 32$ black and white frames of a nonlinear 
pendulum as it evolves for $30$ seconds collecting data at a 
frequency of $15$Hz. 
We transform the $1024$ dimensional state down to two dimensions
using a convolutional architecture.
Details on the hyperparameters and architecture are 
provided in Appendix~\ref{app:arch-nsde-video}.
This problem is similar to the problem considered in~\cite{greydanus_hamiltonian_2019}
except the dynamical system we consider is nonlinear.
In this prior work, the authors were forced to regularize the latent 
space so that one set of coordinates resembles a generalized velocity.
In the present work, no such regularization is required. 

We assume a Bernoulli likelihood on the pixels.
Like in previous numerical studies we place a log-normal 
prior on the diagonals of the diffusion term 
and approximate the posterior using a log-normal variational
distribution.
After training we generate a forecast that is visualized in 
Figure~\ref{fig:neural-sdes-from-video}.
We see that we were successfully able to build a generative 
model for this simple video. 
This result demonstrates the broad applicability 
of the present approach to standard generative modeling tasks. 

\section{Conclusions}
Here we have presented a method for 
constructing unbiased approximations to
gradients of the evidence lower bound used to train latent stochastic differential 
equations with a time and memory cost that scales 
independently with the amount of data, the 
length of the time-series, and the stiffness of the 
model for the latent differential equations.
We achieve this result by trading off the 
numerical precision of adaptive differential equation solvers 
with Monte-Carlo approximations to expectations using 
a novel amortization strategy and a recently derived change of variables 
for expectations under Markov Gaussian processes~\cite{course_state_2023}.

We have demonstrated the efficacy of our approach in learning 
latent SDEs
across a range of test problems. In particular we showed that our approach 
can reduce the number of function evaluations as compared to 
adjoint methods by more than one order of magnitude in training 
while avoiding the numerical instabilities 
of adjoint methods for long time series 
generated from chaotic systems.
In addition, we showed that our approach can be used for generative modeling 
of a simple video.

In the immediate future, there is significant room 
for future work in applying variance reduction schemes 
to the expectation over time to 
further reduce the total number of required 
function evaluations.
There are also opportunities to explore the utility of the proposed
approach for generative modeling on more realistic problems.
Finally, there are opportunities to apply our work in the 
context of implicit densities~\cite{kidger_neural_2021}.

\pagebreak
\begin{ack}
  This research is funded by a grant from NSERC.
\end{ack}

\newpage

\appendix
\makeatletter
\let\thetitle\@title
\makeatother

\begin{center}
  {\Large Appendices} \\
  [1em]
  {\bf \Large \thetitle }
\end{center}

\section{Expectations under linear SDEs}\label{app:identity}
In this section we re-derive a result from~\cite{course_state_2023}
regarding how to rewrite expectations under linear stochastic 
differential equations.
As we will explain in greater detail in Appendix~\ref{app:elbo}, 
it is this result that allowed~\cite{course_state_2023} to rewrite the 
ELBO entirely in terms of quantities that do not require differential
equation solvers.

\begin{theorem}
  Consider the density, $q(z(t))$, of the solution to a linear SDE
  $dz = (-A(t)z + b(t))dt + L(t) d\beta$ 
 with initial condition $z_0\sim\mathcal{N}(m_0, S_0)$,
    where $A: \R\to\R^{d\times d}$ is
    symmetric, $b: \R\to\R^{d}$,
  $L:\R \to \R^{d \times d}$, and $\beta$ indicates Brownian motion 
  with diffusion matrix $\Sigma$.
  Then the expected value of a bounded functional, $f$, satisfies
  \begin{equation}
    \E{f(A(t), b(t), z(t))}{z(t)\sim q(z(t))}\label{eq:general-expectation}
= \E{f(B(t),\dot{m}(t) + B(t) m(t), z(t))}{z(t)\sim\mathcal{N}(m(t), S(t))},
  \end{equation}
where $B(t) = \text{vec}^{-1}((S(t)\oplus S(t))^{-1}\text{vec}(L(t)\Sigma L(t)^T - \dot{S}(t)))$
  and $m(t)$ and $S(t)$ indicate the mean and covariance, respectively, of the
    SDE solution at time $t$.
\end{theorem}
\begin{proof}
The solution of a linear SDE defines a Markov Gaussian process 
with marginal statistics, $q(z(t))\sim \mathcal{N}(m(t), S(t))$,
given by the solution to the ODEs,
\begin{align}
    \dot{m}(t) &= (-A(t)m(t) + b(t)), \label{eq:linear-in-b} \\
    \dot{S}(t) &= -A(t) S(t) - S(t) A(t)^T + L(t) \Sigma L(t)^T 
  \label{eq:lyapunov-in-a},
\end{align}
with initial condition $m(0) = m_0$, $S(0) = S_0$~\cite{sarkka_applied_2019}.
  Noticing that equation \eqref{eq:lyapunov-in-a} defines a set of matrix Lyapunov 
  equations in terms of $A(t)$ allows us to express $A(t)$ as a function 
  of $S(t)$ as follows,
\begin{equation}
  A(t) = \text{vec}^{-1} \left( 
  (S(t) \oplus S(t)  )^{-1}
  \text{vec}(L(t) \Sigma L(t)^T - \dot{S}(t))
  \right), \label{eq:expresion-for-a}
\end{equation}
where $\oplus$ is called the Kronecker sum and is defined 
as $S \oplus S = I \otimes S + S\otimes I$ and $\otimes$ indicates 
the standard Kronecker product.
  Letting $B(t) = A(t)$ be the expression for $A(t)$ written in terms 
  of $S(t)$, we can rearrange Equation~\eqref{eq:linear-in-b} to solve for $b(t)$ as,
  \begin{equation}
    b(t) = \dot{m}(t) + B(t)m(t).
  \end{equation}
  Substituting the expressions for $A(t)$ and $b(t)$ into Equation \eqref{eq:general-expectation}
  yields the desired result.
\end{proof}

\paragraph{Remark }
In the case that 
$S(t)$ and $L(t)\Sigma L(t)^T$ are diagonal as we assume throughout 
this work, Equation~\ref{eq:expresion-for-a}
simplifies to, 
\begin{equation}
  A(t) = \frac{1}{2} S(t)^{-1} (L(t)\Sigma L(t)^T - \dot{S}(t)).
\end{equation}

\section{Evidence lower bound derivation}\label{app:elbo}
Given a dataset of observations,
$
  \mathcal{D} = \left\{(x_i, t_i)\right\}_{i=1}^N
$,
a generative likelihood, $p_\theta(x \mid z(t_i))$
 a prior SDE, $dz = f_\theta(z(t), t)dt + L_\theta(t) d\beta$,
 and an approximation to the posterior of the latent state,
 $dz = (-A_\phi(t)z(t) + b_\phi(t) ) dt + L_\theta(t) d\beta$, 
 it is possible to derive the ELBO~\cite{archambeau_gaussian_2007,li_scalable_2020},
\begin{equation}
  \begin{split}
    \mathcal{L}(\theta, \phi) =& 
    \sum_{i=1}^N \E{\log \p_\theta(x_i \mid z(t_i))}{z(t_i)\sim q_\phi(z(t))} \\
    &-\frac{1}{2} \int_0^T \E{\rnorm{-A_\phi(t)z(t) + b(t) - f_\theta(z(t), t)}}{z(t)\sim q_\phi(z(t))}\,dt, 
    \label{eq:elbo-derivation}
  \end{split}
\end{equation}
where $C_\theta(t) = (L_\theta(t) \Sigma L_\theta^T(t))^{-1}$.
The second term in~\eqref{eq:elbo-derivation} contains an integral over 
time of terms of the form which matches that of the expectation in~\eqref{eq:general-expectation}.
From this observation, Course and Nair~\cite{course_state_2023} applied the identity
in~\eqref{eq:general-expectation} to derive the reparametrized ELBO,
\begin{equation}
  \begin{split}
    \mathcal{L}(\theta, \phi) =& 
    \sum_{i=1}^N \E{\log \p_\theta(x_i \mid z(t_i))}{z(t_i)\sim q_\phi(z(t))} \\
    &-\frac{1}{2} \int_0^T \E{\rnorm{r_{\theta,\phi}(z(t),t)} }{z(t)\sim q_\phi(z(t))}\,dt, 
  \end{split}
\end{equation}
where,
\begin{equation}
  \begin{split}
    C_{\theta}(t) =& (L_\theta(t) \Sigma L_\theta(t)^{T})^{-1} \\
      r_{\theta,\phi}(z(t),t) =& B(t)(m_\phi(t)-z(t))
      + \dot{m}_\phi(t)- f_\theta(z(t), t)\\
      B(t) =& \text{vec}^{-1}\left((S_\phi(t) \oplus S_\phi(t))^{-1}
    \text{vec}(L_\theta(t) \Sigma L_\theta(t)^T - \dot{S}_\phi(t))\right).
  \end{split}
\end{equation}
To reiterate what was mentioned in the main body of the present work, 
the advantage of this reparametrized ELBO is that all expectations 
are now taken with respect to normal distributions -- this has effectively
eliminated the requirement of a differential equation solver.
Unfortunately this approach requires storing the entire estimate 
for the latent state -- making it scale poorly for long time series with 
complex dynamics. 
After amortizing as is suggested in the main body of the present work
we arrive at the final ELBO,
\begin{equation}
  \begin{split}
    \mathcal{L}(\theta, \phi) =&\sum_{j=1}^{N/M}\sum_{i=1}^M \E{\log p_\theta(x_i^{(j)}\mid z(t_i^{(j)}))}
                                                              {q_\phi(z(t_i^{(j)})\mid x_1^{(j)}, \dots, x_M^{(j)})}\\
    &-\frac{1}{2}\int_{t_1^{(j)}}^{t_1^{(j+1)}} \E{\rnorm{r_{\theta,\phi}(z,t)}}
    {q_\phi(z(t)\mid x_{1}^{(j)}, \dots, x_{M}^{(j)})}\,dt.
  \end{split}
\end{equation}
Again, as is discussed at length in the main body of the present 
work, the advantage of such an amortization strategy is that 
it is possible to construct an unbiased approximation to the gradients
of this ELBO that scales independently with the amount of the data, 
the length of the time series, and the stiffness of the underlying differential equations.

\section{Evidence lower bound with priors on generative parameters}\label{app:priors-on-theta}
In many circumstances, it will be advantageous to place priors on a subset 
of the parameters, $\theta$. 
Using the tools of stochastic variational inference, we can infer the posterior 
on these variables with only a marginal increase to the overall computational cost and 
with no impact to the asymptotic computational complexity.

For example, in all numerical studies we placed a log-normal prior 
on the diagonal of the diffusion process term $C_\theta^{-1}(t)$ and approximated 
the posterior using a log-normal approximate posterior, see Appendix~\ref{app:kl-divergence-log-normal} 
for details.
While not necessary, we found that doing so helped to stabilize training 
in the examples we considered. 
The particular choice 
of prior-posterior pair was made to ensure that the Kullback-Leibler (KL) divergence 
between the approximate posterior and the prior could be written in closed form; 
see Appendix~\ref{app:kl-divergence-log-normal}.
Given a particular choice of prior, $p(\theta)$, along with an approximate 
posterior, $q_\phi(\theta)$, we can amend the previously 
derived ELBO as,
\begin{equation}
  \begin{split}
    \mathcal{L}(\phi) =&\sum_{j=1}^{N/M}\sum_{i=1}^M \E{\log p_\theta(x_i^{(j)}\mid z(t_i^{(j)}))}
                                                              {q_\phi(\theta) q_\phi(z(t_i^{(j)})\mid x_1^{(j)}, \dots, x_M^{(j)})}\\
    &-\frac{1}{2}\int_{t_1^{(j)}}^{t_1^{(j+1)}} \E{\rnorm{r_{\theta,\phi}(z,t)}}
    {q_\phi(\theta) q_\phi(z(t)\mid x_{1}^{(j)}, \dots, x_{M}^{(j)})}\,dt - D_{KL}(q_\phi(\theta) \mid \mid p(\theta)),
  \end{split}
\end{equation}
where $D_{KL}(q \mid \mid p)$ indicates the KL divergence between $q$ and $p$.

Consider a reparametrization to the 
latent state as $z(t) = T(t, \epsilon, \phi)$ where 
$\epsilon \sim p(\epsilon) \implies z(t)\sim q_\phi(z(t)\mid x_1^{(j)}, 
x_{2}^{(j)}, \dots, x_{M}^{(j)})$. 
Also consider a reparametrization to the approximate posterior 
for the generative variables, $\theta = V(\nu, \phi)$ where $\nu \sim p(\nu) \implies
\theta \sim q_\phi(\theta)$. Then, building on Lemma~\ref{lemma:unbiased-gradient},
we can arrive at an unbiased estimate for the gradient of the modified 
ELBO which inherits all the properties discussed with the original approximation:
\begin{equation}
  \begin{split}
    \nabla_{\phi}\mathcal{L}(\phi) \approx& \frac{N}{R}\sum_{i=1}^M \sum_{k=1}^{R} 
            \nabla_{\phi}\log p_{V(\nu^{(k)}, \phi)}(x_i^{(j)}\mid T(t_i^{(j)}, \epsilon^{(k)}, \phi))\\
            &- ( t_1^{(j+1)} - t_1^{(j)} )\frac{N}{2R}\sum_{k=1}^R
                      \nabla_{\phi}  \modifiedrnorm{r_{V(\nu^{(k)},\phi),\phi}(T(t^{(k)},\epsilon^{(k)}, \phi),t^{(k)})}\\
            &-D_{KL}(q_\phi(\theta) \mid\mid p(\theta)).
  \end{split}
\end{equation}
where each $t^{(k)} \sim \mathcal{U}(t_1^{(j)}, t_1^{(j+1)})$,
each $\epsilon^{(k)}\sim \mathcal{N}(0, I)$,
and each $\nu^{(k)}\sim p(\nu)$.

\section{Detailed description of recognition network} \label{app:encoder}
This section details the deep kernel based encoder we used 
in all numerical studies. 
We found this particular encoding architecture to be useful 
for our purposes because it is interpretable and stable in training.
With this being said, any encoder which can transform a batch of observations 
down to a reduced dimension latent state can be used in combination 
with Lemma~\ref{lemma:unbiased-gradient} to arrive at an unbiased estimate
for the gradient of the ELBO which retains all the properties discussed in the main body 
of this work.

Given a dataset, $\mathcal{D}=\{(t_i, x_i)\}_{i=1}^N$, recall that the first step in our amortization strategy 
is to split the dataset into $N/M$ non-overlapping partitions,
\begin{alignat*}{7}
  \text{original indexing: }\quad&[t_1,       &&t_2,       &&\dots,       &&t_M, &&t_{M+1}, \dots, &&t_N &&] \\
  \text{reindexed dataset: }\quad&[t_1^{(1)}, &&t_2^{(1)}, &&\dots, &&t_M^{(1)}, &&t_{1}^{(2)}, \dots, &&t_M^{(N/M)} &&] 
\end{alignat*}
Recall that we would like to 
approximate the latent state over each partition using only the $M$ observations
in each partition,
$q_\phi(z(t)\mid x_1^{(j)}, x_{2}^{(j)}, \dots, x_{M}^{(j)}) \approx p(z(t)\mid \mathcal{D })$ 
for $t\in [t_1^{(j)}, t_{1}^{(j+1)}]$.
Going forward we will drop writing the superscript as we will be only working 
on a single partition, $(t_1, x_1), (t_2, x_2), \dots (t_M, x_M)$.
The user first selects how many snapshots into the future they would 
like to use to estimate the latent state at the present time, $K$; in our own 
studies we found choosing $K\in [1, 5]$ worked well for the problems we considered.

Given some encoding network, $\text{ENC}_{\phi}$, we compute:
\begin{equation}
  h_i = \text{ENC}_{\phi}(x_i, x_{i+1}, \dots, x_{i+K}).
\end{equation}
where $h_i\in\R^{2d}$ with $d < D$.
We will describe the particular architecture design for $\text{ENC}_\phi$ 
in the context of each numerical study in Appendix~\ref{app:detailed-exp-design}.
We note that for our approach, it is often important to use at 
least a small number of neighbours (i.e. $K>0$) when estimating the latent 
state because we are limited to approximating MGPs over the latent state.

To explain this point more clearly, let us consider the example of inferring 
a latent SDE using video of a pendulum as we did in Section~\ref{sec:nsde-from-video}.
If we choose a single frame near the center, we have no way 
of knowing from that frame alone if the pendulum is currently swinging left or 
right. In other words, if we were to build an encoder which takes in one 
single frame, the encoder should predict that the posterior given that 
single frame is multimodal. 
As our approach only allows for one to approximate the latent state 
using MGPs, this is not an option.
Allowing the encoder to take in a few frames at a time remedies this issue.
We also note that previous works for inferring latent differential 
equations made this same choice~\cite{chen_neural_2018,rubanova_latent_2019,li_scalable_2020,toth_hamiltonian_2020}.

Recall that we need to approximate the latent state at any time over 
the window of observations in the partition, $t \in [t_1^{(j)}, t_{1}^{(j+1)}]$.
To accomplish this we effectively interpolate between encodings 
using a deep kernel~\cite{wilson_deep_2016}. 
Letting,
\begin{equation}
  H = \begin{bmatrix}
    h_1^T \\ 
     h_2^T \\
    \vdots  \\
    h_M^T 
  \end{bmatrix},
\end{equation}
where $H\in\R^{M\times 2d}$ and $t_{\text{node}} = [t_i, t_1, \dots, t_M]$, we construct 
the encoder for the mean and diagonal of the covariance over the latent 
state as,
\begin{equation}
    \begin{bmatrix}
        m_\phi(t)^T & \log S_\phi(t)^T
    \end{bmatrix}
        = k_\phi(t,
        t_{\text{node}})^T(k_\phi(t_{\text{node}},t_{\text{node}})^{-1} +
        \sigma_n^2 I )^{-1} H.
        \label{eq:app-kernel-interp}
\end{equation}
Here we note that the right hand side of \eqref{eq:app-kernel-interp} is a row
vector of length $2d$. We use the notation $
    \begin{bmatrix}
        m_\phi(t)^T & \log S_\phi(t)^T
    \end{bmatrix}
    $
to indicate that $m_\phi(t)$ is given by the first $d$ elements of this vector 
and $\log S_\phi(t)$ is given by the next $d$ elements.
Here $k_\phi$ is a so-called deep kernel and $k_\phi(t,t_{\text{node}})\in\R^{M}$ and 
$k_\phi(t_{\text{node}},t_{\text{node}})\in\R^{M\times M}$. 
In addition $\sigma_n\in\R^{+}$ is tuned as a part of the optimization procedure.
While many options for the base 
kernel are possible, we made use of a squared exponential kernel, 
\begin{equation}
  k_\phi(t, t_*) = \sigma_f \exp \left( - \frac{\lvert\lvert 
  \text{DK}_\phi(t) - \text{DK}_\phi(t_*)
  \rvert\rvert^2 }{2 \ell^2} \right),
\end{equation}
where $\sigma_f,\ell \in \R^+$ are positive constants tuned as 
a part of the optimization procedure and $\text{DK}_\phi$ is
a neural network whose architecture we will describe 
in the context of the numerical studies.
While the base kernel is stationary, the neural networks 
allow for the encoder to infer non-stationary relationships~\cite{wilson_deep_2016}.
It is worth noting that without this deep-kernel our approach struggled 
to achieve good validation accuracy on the datasets we considered.

Advantages of this encoder design are that it can easily 
take in varying amounts of data, it is interpretable because 
$[m_\phi(t_i)^T \log S_\phi(t_i)^T ] \approx h_i^T$, and it is cheap to compute 
so long as $M$ is relatively small because evaluations of $\text{ENC}_{\phi}$
can be performed in parallel.
Particular choices for $\text{ENC}_\phi$ and $\text{DK}_\phi$ are  
described in context in Appendix~\ref{app:detailed-exp-design}.

\section{Approximate posterior on diffusion term}\label{app:kl-divergence-log-normal}
This section summarizes the log-normal parameterization we used to approximate the
posterior over the diagonal drift function terms in the main body of the paper.
First, we note that in the loss function we only require access to the product, 
$ C_\theta(t)^{-1} = L_\theta(t) \Sigma L_\theta(t)^T$ so, rather than parametrizing 
$L_\theta$ on its own, we parametrize $C_\theta^{-1}$.

Specifically we parametrize $C_\theta(t)^{-1} = \text{diag}(\theta)$,
where $\theta \in\R^d$.
The prior is defined as,
\begin{align}
    p(\theta) = \prod_{i=1}^d \mathcal{LN}(\theta_i \mid \tilde{\mu}_i, \tilde{\sigma}_i^2).
\end{align}
We parametrize the approximate posterior as,
\begin{align}
    q_\phi(\theta) = \prod_{i=1}^d \mathcal{LN}(\theta_i \mid \mu_i, \sigma_i^2),
\end{align}
where $\mu_i$ and $\sigma_i$ are the variational parameters. The KL divergence
between the posterior and prior is given by,
\begin{align}
    \KL{q_\phi(\theta)}{\p(\theta)} = \sum_{i=1}^d (\log \tilde{\sigma}_i - \log \sigma_i )
    -\frac{1}{2}\left( d 
    -\sum_{i=1}^d\frac{\sigma_i^2 - (\mu_i-\tilde{\mu}_i)^2}{\tilde{\sigma}_i^2}  \right).
\end{align}

\section{Computing resources}\label{app:computing-resources}
Experiments were performed on an Ubuntu server
with a 
dual E5-2680 v3 with a total of 24 cores,
128GB of RAM, 
and an NVIDIA GeForce RTX 4090 GPU.
The majority of our code is written in PyTorch~\cite{paszke_pytorch_2019}.
For benchmarking we made use of \verb|torchdiffeq|~\cite{chen_neural_2018},
\verb|torchsde|~\cite{li_scalable_2020, kidger_neural_2021}, 
and \verb|pytorch_lightning|.
All code is available at
\href{https://github.com/coursekevin/arlatentsde}{github.com/coursekevin/arlatentsde}.

\section{Details on numerical studies}\label{app:detailed-exp-design}
This section contains more details on the numerical studies including 
the specific architecture design and hyperparameter selection for each
experiment. For each experiment we design an encoder consisting of 
two neural networks, $\text{ENC}_{\phi}(x_i,\dots, x_{i+K} )$ 
and $\text{DK}_\phi(t)$ (see Appendix~\ref{app:encoder}),
a decoder $\text{DEC}_\theta(z(t))$,
and a model for the SDE consisting of a drift, $f_\theta(t, z)$, and dispersion 
matrix, $L_\theta(t)$.
For all numerical studies we place a log-normal prior on the dispersion matrix 
and assume that the approximate posterior is constant in time, 
see Appendix~\ref{app:kl-divergence-log-normal}.
For all experiments we gradually increased the value of the KL-divergence (both the 
KL-divergence due to the SDE prior and the KL-divergence on the dispersion matrix parameters)
from $0$ to $1$ using a linear schedule.

\subsection{Orders of magnitude fewer function evaluations}\label{app:arch-lotka-volterra}
In this section we provide a more detailed description of the numerical
study in Section~\ref{sec:lotka-volterra}.
As a reminder, we tasked a neural ODE (NODE) and our approach with building a predictive 
model for the Lotka-Volterra system given a dataset of time-series observations. 
The Lotka-Volterra equations are a system of 
nonlinear ODEs usually written as,
\begin{equation}
  \begin{split}
    \dot{x} &= \alpha x - \beta x y, \\
    \dot{y} &= \delta x y - \gamma y.
  \end{split}
\end{equation}
In our experiment we chose $\alpha = 2/3$, $\beta=4/3$, and $\delta=\gamma=1$.
We also assumed that there was some small amount of Brownian noise
given by $\Sigma = \text{diag}(10^{-3}, 10^{-3})$.
Using the initial condition $x = 0.9$ and $y=0.2$, we draw a sample 
from the system using 
the default adaptive stepping scheme in 
\verb|torchsde| from time $0$ to $65$ with an initial step size of $0.1$
and an absolute and relative tolerance of $10^{-5}$.
We then evaluate the solution at a frequency of $50$Hz and added Gaussian noise 
with a standard deviation of $0.01$. 
We use the 
first $50$ seconds for training and reserve the remaining $15$ seconds
for validation.

Both the NODE and our approach use the same model for the gradient 
field and drift function respectively, see Figure~\ref{fig:lotka-volterra-drift-arch}. 
The encoder and deep kernel architecture are provided in Figures~\ref{fig:lotka-volterra-enc-arch}
and~\ref{fig:lotka-volterra-dk-arch} respectively.
As mentioned in the main body of the paper, we set the decoder to be the identity 
function so as to force our model to learn the dynamics in the original coordinates.
We selected a Gaussian likelihood with a constant standard deviation of $0.01$.
\begin{figure}[htbp]
  \centering
  \begin{subfigure}[t]{0.3\textwidth}
  \centering
    \begin{tikzpicture}[
    layer/.style={
        rectangle,
        draw=black,
        very thick,
        minimum width=2cm,
        align=center,
        rounded corners
    },
    arrow/.style={
        ->,
        >=latex,
        very thick
    }
]

\node[layer,fill={rgb,1:red,0.5529411764705883;green,0.8274509803921568;blue,0.7803921568627451},draw={rgb,1:red,0.49411764705882355;green,0.7411764705882353;blue,0.7019607843137254}] (x00) at (-1.15,0.0) {$x$};
\node[layer,fill={rgb,1:red,1.0;green,1.0;blue,0.7019607843137254},draw={rgb,1:red,0.8980392156862745;green,0.8980392156862745;blue,0.6313725490196078}] (FC6410) at (-1.15,-1.0) {FC-64};
\node[layer,fill={rgb,1:red,1.0;green,1.0;blue,0.7019607843137254},draw={rgb,1:red,0.8980392156862745;green,0.8980392156862745;blue,0.6313725490196078}] (ReLU20) at (-1.15,-2.0) {ReLU};
\node[layer,fill={rgb,1:red,1.0;green,1.0;blue,0.7019607843137254},draw={rgb,1:red,0.8980392156862745;green,0.8980392156862745;blue,0.6313725490196078}] (FC6430) at (-1.15,-3.0) {FC-64};
\node[layer,fill={rgb,1:red,1.0;green,1.0;blue,0.7019607843137254},draw={rgb,1:red,0.8980392156862745;green,0.8980392156862745;blue,0.6313725490196078}] (ReLU40) at (-1.15,-4.0) {ReLU};
\node[layer,fill={rgb,1:red,1.0;green,1.0;blue,0.7019607843137254},draw={rgb,1:red,0.8980392156862745;green,0.8980392156862745;blue,0.6313725490196078}] (FC6450) at (-1.15,-5.0) {FC-64};
\node[layer,fill={rgb,1:red,1.0;green,1.0;blue,0.7019607843137254},draw={rgb,1:red,0.8980392156862745;green,0.8980392156862745;blue,0.6313725490196078}] (ReLU60) at (-1.15,-6.0) {ReLU};
\node[layer,fill={rgb,1:red,1.0;green,1.0;blue,0.7019607843137254},draw={rgb,1:red,0.8980392156862745;green,0.8980392156862745;blue,0.6313725490196078}] (FC270) at (-1.15,-7.0) {FC-2};
\node[layer,fill={rgb,1:red,0.7450980392156863;green,0.7294117647058823;blue,0.8549019607843137},draw={rgb,1:red,0.6705882352941176;green,0.6549019607843137;blue,0.7686274509803922}] (driftfunction80) at (-1.15,-8.0) {drift function};
\draw[arrow] (x00) -- (FC6410);
\draw[arrow] (FC6410) -- (ReLU20);
\draw[arrow] (ReLU20) -- (FC6430);
\draw[arrow] (FC6430) -- (ReLU40);
\draw[arrow] (ReLU40) -- (FC6450);
\draw[arrow] (FC6450) -- (ReLU60);
\draw[arrow] (ReLU60) -- (FC270);
\draw[arrow] (FC270) -- (driftfunction80);\end{tikzpicture}
    \caption{$f_\theta(t, x)$}
  \label{fig:lotka-volterra-drift-arch}
  \end{subfigure}\hfill
  \begin{subfigure}[t]{0.4\textwidth}
  \centering
    \begin{tikzpicture}[
    layer/.style={
        rectangle,
        draw=black,
        very thick,
        minimum width=2cm,
        align=center,
        rounded corners
    },
    arrow/.style={
        ->,
        >=latex,
        very thick
    }
]

\node[layer,fill={rgb,1:red,0.5529411764705883;green,0.8274509803921568;blue,0.7803921568627451},draw={rgb,1:red,0.49411764705882355;green,0.7411764705882353;blue,0.7019607843137254}] (x_i00) at (-1.15,0.0) {$x_i$};
\node[layer,fill={rgb,1:red,1.0;green,1.0;blue,0.7019607843137254},draw={rgb,1:red,0.8980392156862745;green,0.8980392156862745;blue,0.6313725490196078}] (FC3210) at (-1.15,-1.0) {FC-32};
\node[layer,fill={rgb,1:red,1.0;green,1.0;blue,0.7019607843137254},draw={rgb,1:red,0.8980392156862745;green,0.8980392156862745;blue,0.6313725490196078}] (ReLU20) at (-1.15,-2.0) {ReLU};
\node[layer,fill={rgb,1:red,1.0;green,1.0;blue,0.7019607843137254},draw={rgb,1:red,0.8980392156862745;green,0.8980392156862745;blue,0.6313725490196078}] (FC3230) at (-1.15,-3.0) {FC-32};
\node[layer,fill={rgb,1:red,1.0;green,1.0;blue,0.7019607843137254},draw={rgb,1:red,0.8980392156862745;green,0.8980392156862745;blue,0.6313725490196078}] (ReLU40) at (-1.15,-4.0) {ReLU};
\node[layer,fill={rgb,1:red,1.0;green,1.0;blue,0.7019607843137254},draw={rgb,1:red,0.8980392156862745;green,0.8980392156862745;blue,0.6313725490196078}] (FC450) at (-1.15,-5.0) {FC-4};
\node[layer,fill={rgb,1:red,0.7450980392156863;green,0.7294117647058823;blue,0.8549019607843137},draw={rgb,1:red,0.6705882352941176;green,0.6549019607843137;blue,0.7686274509803922}] (mt_i60) at (-2.3,-6.0) {$m(t_i)$};
\node[layer,fill={rgb,1:red,0.7450980392156863;green,0.7294117647058823;blue,0.8549019607843137},draw={rgb,1:red,0.6705882352941176;green,0.6549019607843137;blue,0.7686274509803922}] (logSt_i61) at (0.0,-6.0) {$\log S(t_i)$};
\draw[arrow] (x_i00) -- (FC3210);
\draw[arrow] (FC3210) -- (ReLU20);
\draw[arrow] (ReLU20) -- (FC3230);
\draw[arrow] (FC3230) -- (ReLU40);
\draw[arrow] (ReLU40) -- (FC450);
\draw[arrow] (FC450) -- (mt_i60);
\draw[arrow] (FC450) -- (logSt_i61);\draw[arrow] (x_i00) to [out=200,in=110] (mt_i60);\end{tikzpicture}
    \caption{$\text{ENC}_\phi(x_i)$}
  \label{fig:lotka-volterra-enc-arch}
  \end{subfigure}\hfill
  \begin{subfigure}[t]{0.3\textwidth}
  \centering
    \begin{tikzpicture}[
    layer/.style={
        rectangle,
        draw=black,
        very thick,
        minimum width=2cm,
        align=center,
        rounded corners
    },
    arrow/.style={
        ->,
        >=latex,
        very thick
    }
]

\node[layer,fill={rgb,1:red,0.5529411764705883;green,0.8274509803921568;blue,0.7803921568627451},draw={rgb,1:red,0.49411764705882355;green,0.7411764705882353;blue,0.7019607843137254}] (t00) at (-1.15,0.0) {$t$};
\node[layer,fill={rgb,1:red,1.0;green,1.0;blue,0.7019607843137254},draw={rgb,1:red,0.8980392156862745;green,0.8980392156862745;blue,0.6313725490196078}] (FC3210) at (-1.15,-1.0) {FC-32};
\node[layer,fill={rgb,1:red,1.0;green,1.0;blue,0.7019607843137254},draw={rgb,1:red,0.8980392156862745;green,0.8980392156862745;blue,0.6313725490196078}] (ReLU20) at (-1.15,-2.0) {ReLU};
\node[layer,fill={rgb,1:red,1.0;green,1.0;blue,0.7019607843137254},draw={rgb,1:red,0.8980392156862745;green,0.8980392156862745;blue,0.6313725490196078}] (FC3230) at (-1.15,-3.0) {FC-32};
\node[layer,fill={rgb,1:red,1.0;green,1.0;blue,0.7019607843137254},draw={rgb,1:red,0.8980392156862745;green,0.8980392156862745;blue,0.6313725490196078}] (ReLU40) at (-1.15,-4.0) {ReLU};
\node[layer,fill={rgb,1:red,1.0;green,1.0;blue,0.7019607843137254},draw={rgb,1:red,0.8980392156862745;green,0.8980392156862745;blue,0.6313725490196078}] (FC150) at (-1.15,-5.0) {FC-1};
\node[layer,fill={rgb,1:red,0.7450980392156863;green,0.7294117647058823;blue,0.8549019607843137},draw={rgb,1:red,0.6705882352941176;green,0.6549019607843137;blue,0.7686274509803922}] (nonstationarytransform60) at (-1.15,-6.0) {nonstationary transform};
\draw[arrow] (t00) -- (FC3210);
\draw[arrow] (FC3210) -- (ReLU20);
\draw[arrow] (ReLU20) -- (FC3230);
\draw[arrow] (FC3230) -- (ReLU40);
\draw[arrow] (ReLU40) -- (FC150);
\draw[arrow] (FC150) -- (nonstationarytransform60);\end{tikzpicture}
    \caption{$\text{DK}_\phi(t)$}
  \label{fig:lotka-volterra-dk-arch}
  \end{subfigure}
  \caption{
  Architecture diagrams for the drift, deep kernel, and encoder
  used in the Lotka-Volterra problem 
  are provided in Figures~(a), (b), and (c) respectively.
  Note that we have used the shorthand $m(t_i)$, $\log S(t_i)$ 
  to show how we have split the columns of $h_i$ in two. The value 
  of $[m(t_i), \log S(t_i)]$ only $\approx h_i$ unless $\sigma_n = 0$, see Appendix~\ref{app:encoder}.
  Note the arrow from $x_i$ to $m(t_i)$ indicates 
  a residual connection (which was useful in this case because we are learning
  a SDE in the original data coordinates).}
  
\end{figure}
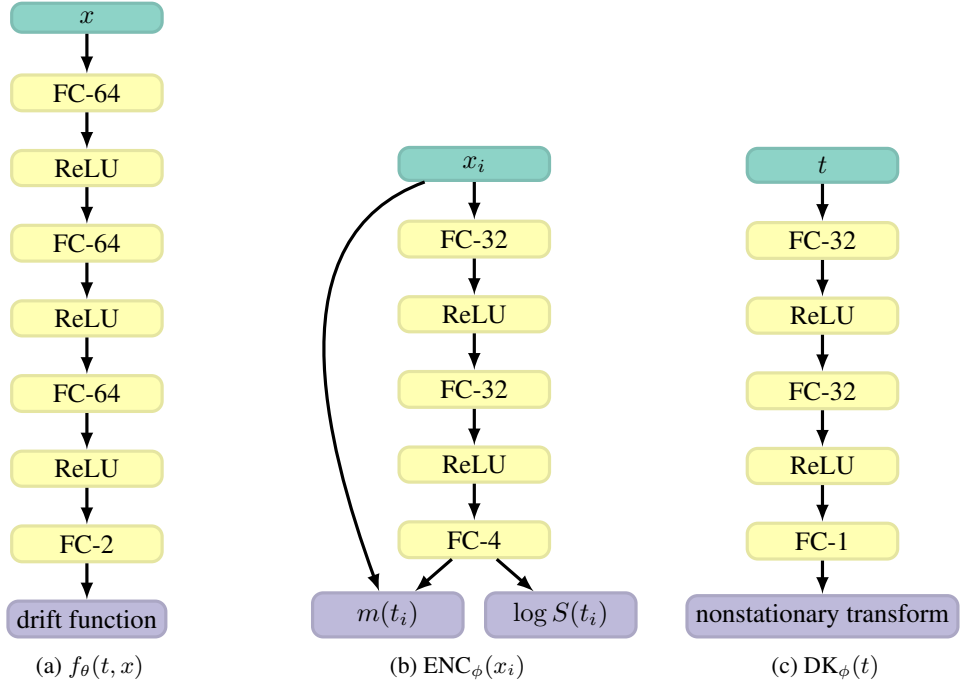

In terms of hyperparameters we set the schedule on the KL-divergence to increase 
from $0$ to $1$ over $1000$ iterations.
We choose a learning rate of $10^{-3}$ with 
exponential learning rate decay where the learning rate 
was decayed $lr = \gamma lr$ every iteration with $\gamma = 
       \exp(\log(0.9) / 1000)$ (i.e. the effective rate of learning rate decay is 
       $lr=0.9lr$ every $1000$ iterations.).
We used the nested Monte-Carlo approximation described in Equation~\eqref{eq:nested-monte-carlo}
with $R=1$, $S=10$, and $M=256$.
In terms of kernel parameters, we initialized $\ell=10^{-2}$, $\sigma_f = 1$, and $\sigma_n = 10^{-5}$.
In terms of the diffusion term, set ${\mu}_i = \sigma_i = 10^{-5}$ and $\tilde{\mu}_i=\tilde{\sigma}_i = 1$.

\subsection{Adjoint instabilities experiment} \label{app:arch-adjoints}
In this section, we provide some additional details of the
numerical study described in Section~\ref{sec:adjoints}.
Recall the parametric model for the Lorenz system in equations~(\ref{eq:dx-lorenz}--\ref{eq:dz-lorenz}).
As a reminder, given time-series dataset of observations, our goal was to infer the value 
of the parameters, $\sigma$, $\beta$, $\rho$, which were likely to have generated the data 
starting from an initial guess: $\theta_0 = [\sigma_0$, $\beta_0$, $\rho_0]$.
The true value of the parameters was chosen as $\theta_*=[10,8/3,28]$.
For all experiments we used the initial condition $[8,-2,36.05]$ as was
suggested in~\cite{lea_sensitivity_2000}.
We generated data by solving the differential equation using \verb|scipy|'s
RK4(5) initial value problem solver with a relative and absolute tolerance of $10^{-6}$ and $10^{-8}$ 
respectively.
We generated data at a frequency of $200$Hz
over the time intervals
$[0, 1]$, $[0, 10]$, $[0, 50]$, and $[0, 100]$.
For each time interval we generated 5 datasets by adding independent Gaussian 
noise with a variance of 1 to the data.
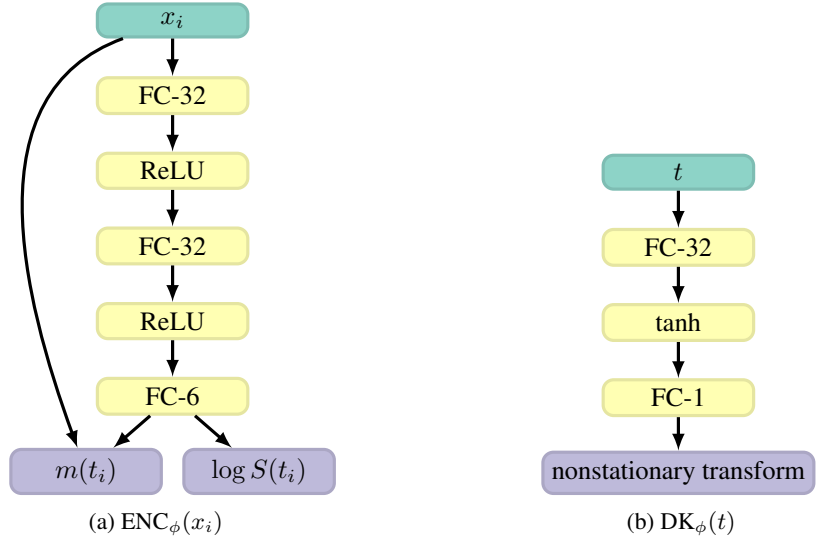
\begin{figure}[htbp]
  \centering
  \begin{subfigure}[t]{0.5\textwidth}
  \centering
    \begin{tikzpicture}[
    layer/.style={
        rectangle,
        draw=black,
        very thick,
        minimum width=2cm,
        align=center,
        rounded corners
    },
    arrow/.style={
        ->,
        >=latex,
        very thick
    }
]

\node[layer,fill={rgb,1:red,0.5529411764705883;green,0.8274509803921568;blue,0.7803921568627451},draw={rgb,1:red,0.49411764705882355;green,0.7411764705882353;blue,0.7019607843137254}] (x_i00) at (-1.15,0.0) {$x_i$};
\node[layer,fill={rgb,1:red,1.0;green,1.0;blue,0.7019607843137254},draw={rgb,1:red,0.8980392156862745;green,0.8980392156862745;blue,0.6313725490196078}] (FC3210) at (-1.15,-1.0) {FC-32};
\node[layer,fill={rgb,1:red,1.0;green,1.0;blue,0.7019607843137254},draw={rgb,1:red,0.8980392156862745;green,0.8980392156862745;blue,0.6313725490196078}] (ReLU20) at (-1.15,-2.0) {ReLU};
\node[layer,fill={rgb,1:red,1.0;green,1.0;blue,0.7019607843137254},draw={rgb,1:red,0.8980392156862745;green,0.8980392156862745;blue,0.6313725490196078}] (FC3230) at (-1.15,-3.0) {FC-32};
\node[layer,fill={rgb,1:red,1.0;green,1.0;blue,0.7019607843137254},draw={rgb,1:red,0.8980392156862745;green,0.8980392156862745;blue,0.6313725490196078}] (ReLU40) at (-1.15,-4.0) {ReLU};
\node[layer,fill={rgb,1:red,1.0;green,1.0;blue,0.7019607843137254},draw={rgb,1:red,0.8980392156862745;green,0.8980392156862745;blue,0.6313725490196078}] (FC650) at (-1.15,-5.0) {FC-6};
\node[layer,fill={rgb,1:red,0.7450980392156863;green,0.7294117647058823;blue,0.8549019607843137},draw={rgb,1:red,0.6705882352941176;green,0.6549019607843137;blue,0.7686274509803922}] (mt_i60) at (-2.3,-6.0) {$m(t_i)$};
\node[layer,fill={rgb,1:red,0.7450980392156863;green,0.7294117647058823;blue,0.8549019607843137},draw={rgb,1:red,0.6705882352941176;green,0.6549019607843137;blue,0.7686274509803922}] (logSt_i61) at (0.0,-6.0) {$\log S(t_i)$};
\draw[arrow] (x_i00) -- (FC3210);
\draw[arrow] (FC3210) -- (ReLU20);
\draw[arrow] (ReLU20) -- (FC3230);
\draw[arrow] (FC3230) -- (ReLU40);
\draw[arrow] (ReLU40) -- (FC650);
\draw[arrow] (FC650) -- (mt_i60);
\draw[arrow] (FC650) -- (logSt_i61);\draw[arrow] (x_i00) to [out=200,in=110] (mt_i60);\end{tikzpicture}
    \caption{$\text{ENC}_\phi(x_i)$}
  \label{fig:shadowing-enc-arch}
  \end{subfigure}\hfill
  \begin{subfigure}[t]{0.5\textwidth}
  \centering
    \begin{tikzpicture}[
    layer/.style={
        rectangle,
        draw=black,
        very thick,
        minimum width=2cm,
        align=center,
        rounded corners
    },
    arrow/.style={
        ->,
        >=latex,
        very thick
    }
]

\node[layer,fill={rgb,1:red,0.5529411764705883;green,0.8274509803921568;blue,0.7803921568627451},draw={rgb,1:red,0.49411764705882355;green,0.7411764705882353;blue,0.7019607843137254}] (t00) at (-1.15,0.0) {$t$};
\node[layer,fill={rgb,1:red,1.0;green,1.0;blue,0.7019607843137254},draw={rgb,1:red,0.8980392156862745;green,0.8980392156862745;blue,0.6313725490196078}] (FC3210) at (-1.15,-1.0) {FC-32};
\node[layer,fill={rgb,1:red,1.0;green,1.0;blue,0.7019607843137254},draw={rgb,1:red,0.8980392156862745;green,0.8980392156862745;blue,0.6313725490196078}] (tanh20) at (-1.15,-2.0) {tanh};
\node[layer,fill={rgb,1:red,1.0;green,1.0;blue,0.7019607843137254},draw={rgb,1:red,0.8980392156862745;green,0.8980392156862745;blue,0.6313725490196078}] (FC130) at (-1.15,-3.0) {FC-1};
\node[layer,fill={rgb,1:red,0.7450980392156863;green,0.7294117647058823;blue,0.8549019607843137},draw={rgb,1:red,0.6705882352941176;green,0.6549019607843137;blue,0.7686274509803922}] (nonstationarytransform40) at (-1.15,-4.0) {nonstationary transform};
\draw[arrow] (t00) -- (FC3210);
\draw[arrow] (FC3210) -- (tanh20);
\draw[arrow] (tanh20) -- (FC130);
\draw[arrow] (FC130) -- (nonstationarytransform40);\end{tikzpicture}
    \caption{$\text{DK}_\phi(t)$}
  \label{fig:shadowing-dk-arch}
  \end{subfigure}
  \caption{
  Architecture diagrams for the deep kernel and encoder
  used in the Lorenz system parameter tuning problem
  are provided in Figures~(a) and (b) respectively.
  Note that we have used the shorthand $m(t_i)$, $\log S(t_i)$ 
  to show how we have split the columns of $h_i$ in two. The value 
  of $[m(t_i), \log S(t_i)]$ only $\approx h_i$ unless $\sigma_n = 0$, see Appendix~\ref{app:encoder}.
  Note the arrow from $x_i$ to $m(t_i)$ indicates 
  a residual connection (which was useful in this case because we are learning
  a SDE in the original data coordinates).}
\label{fig:lorenz-enc-architecture}
\end{figure}

To arrive at an initial guess we sample from the distribution
$\theta_0 \sim \mathcal{N}(\theta_*, (0.2\theta_*)^2 )$.
For each time series length we tasked our approach with
inferring the true value of the parameters given 5 different guesses for 
the initial condition (i.e. one guess / dataset). The reported gradients for our approach 
are given by the average $\ell_2$-norm of the gradient of the ELBO with respect to the 
parameters $\sigma$, $\beta$, and $\rho$ after optimizing for 2000 iterations.
For the adjoint method, we report the gradient of the function:
\begin{equation}
  \mathcal{L}(\theta) = \frac{1}{3N}\sum_{i=1}^N (x_i - x_\theta(t_i))^2 + (y_i - y_\theta(t_i))^2
   + (z_i - z_\theta(t_i))^2,
\end{equation}
at the starting iteration. 
Note we
cannot provide average gradients over the entire optimization procedure for the 
adjoint based method
because the initial gradients are too large when the time interval was $[0,50]$ 
or $[0,100]$.
The error bars are given by one standard deviation from the mean.

A description of the encoder architecture is provided in Figure~\ref{fig:lorenz-enc-architecture}.
As was the case in the previous experiment, we set the decoder to be the 
identity function and placed a log-normal prior on the diffusion term.

In terms of hyperparameters, we choose $M=128$, $R=10$, and $S=100$.
We selected a learning rate of $0.1$ and
decayed the learning rate $lr = \gamma lr$ every iteration with $\gamma = 
       \exp(\log(0.9) / 1000)$.
We linearly increased the value of the KL-divergence from $0$ to $1$ over 
the course of 100 iterations.
In terms of kernel parameters we initialized $\sigma_f=1$, $\ell=10^{-2}$, 
and $\sigma_n = 10^{-5}$. 
For the diffusion term, we set ${\mu}_i = \sigma_i = 10^{-5}$ and $\tilde{\mu}_i=\tilde{\sigma}_i = 10^{-5}$.

\subsection{MOCAP experiment} \label{app:arch-mocap}
For this experiment we 
use the preprocessed dataset provided by~\cite{yildiz_ode2vae_2019}.

Like in previous examples, we place a log-normal prior on the diffusion term.
Like previous works making use of the benchmark, we assume a Gaussian likelihood. 
We place a log-normal prior on the variance of the likelihood.
A description of the architecture is provided in Figure~\ref{fig:mocap-architecture}.
Note that the architecture we chose is very similar to the architecture used 
by~\cite{yildiz_ode2vae_2019,li_scalable_2020}.
\begin{figure}[hp]
  \centering
  \begin{subfigure}[t]{0.45\textwidth}
  \centering
    \begin{tikzpicture}[
    layer/.style={
        rectangle,
        draw=black,
        very thick,
        minimum width=2cm,
        align=center,
        rounded corners
    },
    arrow/.style={
        ->,
        >=latex,
        very thick
    }
]

\node[layer,fill={rgb,1:red,0.5529411764705883;green,0.8274509803921568;blue,0.7803921568627451},draw={rgb,1:red,0.49411764705882355;green,0.7411764705882353;blue,0.7019607843137254}] (x_i00) at (-3.4499999999999997,0.0) {$x_i$};
\node[layer,fill={rgb,1:red,0.5529411764705883;green,0.8274509803921568;blue,0.7803921568627451},draw={rgb,1:red,0.49411764705882355;green,0.7411764705882353;blue,0.7019607843137254}] (x_i101) at (-1.15,0.0) {$x_{i+1}$};
\node[layer,fill={rgb,1:red,0.5529411764705883;green,0.8274509803921568;blue,0.7803921568627451},draw={rgb,1:red,0.49411764705882355;green,0.7411764705882353;blue,0.7019607843137254}] (x_i202) at (1.15,0.0) {$x_{i+2}$};
\node[layer,fill={rgb,1:red,1.0;green,1.0;blue,0.7019607843137254},draw={rgb,1:red,0.8980392156862745;green,0.8980392156862745;blue,0.6313725490196078}] (FC3010) at (-1.15,-1.0) {FC-30};
\node[layer,fill={rgb,1:red,1.0;green,1.0;blue,0.7019607843137254},draw={rgb,1:red,0.8980392156862745;green,0.8980392156862745;blue,0.6313725490196078}] (tanh20) at (-1.15,-2.0) {tanh};
\node[layer,fill={rgb,1:red,1.0;green,1.0;blue,0.7019607843137254},draw={rgb,1:red,0.8980392156862745;green,0.8980392156862745;blue,0.6313725490196078}] (FC3030) at (-1.15,-3.0) {FC-30};
\node[layer,fill={rgb,1:red,1.0;green,1.0;blue,0.7019607843137254},draw={rgb,1:red,0.8980392156862745;green,0.8980392156862745;blue,0.6313725490196078}] (tanh40) at (-1.15,-4.0) {tanh};
\node[layer,fill={rgb,1:red,1.0;green,1.0;blue,0.7019607843137254},draw={rgb,1:red,0.8980392156862745;green,0.8980392156862745;blue,0.6313725490196078}] (FC1550) at (-1.15,-5.0) {FC-15};
\node[layer,fill={rgb,1:red,1.0;green,1.0;blue,0.7019607843137254},draw={rgb,1:red,0.8980392156862745;green,0.8980392156862745;blue,0.6313725490196078}] (tanh60) at (-1.15,-6.0) {tanh};
\node[layer,fill={rgb,1:red,1.0;green,1.0;blue,0.7019607843137254},draw={rgb,1:red,0.8980392156862745;green,0.8980392156862745;blue,0.6313725490196078}] (FC670) at (-1.15,-7.0) {FC-6};
\node[layer,fill={rgb,1:red,0.7450980392156863;green,0.7294117647058823;blue,0.8549019607843137},draw={rgb,1:red,0.6705882352941176;green,0.6549019607843137;blue,0.7686274509803922}] (h_i80) at (-1.15,-8.0) {$h_i$};
\draw[arrow] (x_i00) -- (FC3010);
\draw[arrow] (x_i101) -- (FC3010);
\draw[arrow] (x_i202) -- (FC3010);
\draw[arrow] (FC3010) -- (tanh20);
\draw[arrow] (tanh20) -- (FC3030);
\draw[arrow] (FC3030) -- (tanh40);
\draw[arrow] (tanh40) -- (FC1550);
\draw[arrow] (FC1550) -- (tanh60);
\draw[arrow] (tanh60) -- (FC670);
\draw[arrow] (FC670) -- (h_i80);\end{tikzpicture}
    \caption{$\text{ENC}_\phi(x_i)$}
  \label{fig:mocap-enc-arch}
  \end{subfigure}\hfill
  \begin{subfigure}[t]{0.45\textwidth}
  \centering
    \begin{tikzpicture}[
    layer/.style={
        rectangle,
        draw=black,
        very thick,
        minimum width=2cm,
        align=center,
        rounded corners
    },
    arrow/.style={
        ->,
        >=latex,
        very thick
    }
]

\node[layer,fill={rgb,1:red,0.5529411764705883;green,0.8274509803921568;blue,0.7803921568627451},draw={rgb,1:red,0.49411764705882355;green,0.7411764705882353;blue,0.7019607843137254}] (zt00) at (-1.15,0.0) {$z(t)$};
\node[layer,fill={rgb,1:red,1.0;green,1.0;blue,0.7019607843137254},draw={rgb,1:red,0.8980392156862745;green,0.8980392156862745;blue,0.6313725490196078}] (FC1510) at (-1.15,-1.0) {FC-15};
\node[layer,fill={rgb,1:red,1.0;green,1.0;blue,0.7019607843137254},draw={rgb,1:red,0.8980392156862745;green,0.8980392156862745;blue,0.6313725490196078}] (tanh20) at (-1.15,-2.0) {tanh};
\node[layer,fill={rgb,1:red,1.0;green,1.0;blue,0.7019607843137254},draw={rgb,1:red,0.8980392156862745;green,0.8980392156862745;blue,0.6313725490196078}] (FC3030) at (-1.15,-3.0) {FC-30};
\node[layer,fill={rgb,1:red,1.0;green,1.0;blue,0.7019607843137254},draw={rgb,1:red,0.8980392156862745;green,0.8980392156862745;blue,0.6313725490196078}] (tanh40) at (-1.15,-4.0) {tanh};
\node[layer,fill={rgb,1:red,1.0;green,1.0;blue,0.7019607843137254},draw={rgb,1:red,0.8980392156862745;green,0.8980392156862745;blue,0.6313725490196078}] (FC3050) at (-1.15,-5.0) {FC-30};
\node[layer,fill={rgb,1:red,1.0;green,1.0;blue,0.7019607843137254},draw={rgb,1:red,0.8980392156862745;green,0.8980392156862745;blue,0.6313725490196078}] (tanh60) at (-1.15,-6.0) {tanh};
\node[layer,fill={rgb,1:red,1.0;green,1.0;blue,0.7019607843137254},draw={rgb,1:red,0.8980392156862745;green,0.8980392156862745;blue,0.6313725490196078}] (FC5070) at (-1.15,-7.0) {FC-50};
\node[layer,fill={rgb,1:red,0.7450980392156863;green,0.7294117647058823;blue,0.8549019607843137},draw={rgb,1:red,0.6705882352941176;green,0.6549019607843137;blue,0.7686274509803922}] (mut80) at (-1.15,-8.0) {$\mu(t)$};
\draw[arrow] (zt00) -- (FC1510);
\draw[arrow] (FC1510) -- (tanh20);
\draw[arrow] (tanh20) -- (FC3030);
\draw[arrow] (FC3030) -- (tanh40);
\draw[arrow] (tanh40) -- (FC3050);
\draw[arrow] (FC3050) -- (tanh60);
\draw[arrow] (tanh60) -- (FC5070);
\draw[arrow] (FC5070) -- (mut80);\end{tikzpicture}
    \caption{$\text{DEC}_\phi(t)$}
  \label{fig:mocap-dec-arch}
  \end{subfigure}\\
  \begin{subfigure}[b]{0.45\textwidth}
  \centering
    \begin{tikzpicture}[
    layer/.style={
        rectangle,
        draw=black,
        very thick,
        minimum width=2cm,
        align=center,
        rounded corners
    },
    arrow/.style={
        ->,
        >=latex,
        very thick
    }
]

\node[layer,fill={rgb,1:red,0.5529411764705883;green,0.8274509803921568;blue,0.7803921568627451},draw={rgb,1:red,0.49411764705882355;green,0.7411764705882353;blue,0.7019607843137254}] (t00) at (-1.15,0.0) {$t$};
\node[layer,fill={rgb,1:red,1.0;green,1.0;blue,0.7019607843137254},draw={rgb,1:red,0.8980392156862745;green,0.8980392156862745;blue,0.6313725490196078}] (FC3210) at (-1.15,-1.0) {FC-32};
\node[layer,fill={rgb,1:red,1.0;green,1.0;blue,0.7019607843137254},draw={rgb,1:red,0.8980392156862745;green,0.8980392156862745;blue,0.6313725490196078}] (tanh20) at (-1.15,-2.0) {tanh};
\node[layer,fill={rgb,1:red,1.0;green,1.0;blue,0.7019607843137254},draw={rgb,1:red,0.8980392156862745;green,0.8980392156862745;blue,0.6313725490196078}] (FC130) at (-1.15,-3.0) {FC-1};
\node[layer,fill={rgb,1:red,0.7450980392156863;green,0.7294117647058823;blue,0.8549019607843137},draw={rgb,1:red,0.6705882352941176;green,0.6549019607843137;blue,0.7686274509803922}] (nonstationarytransform40) at (-1.15,-4.0) {nonstationary transform};
\draw[arrow] (t00) -- (FC3210);
\draw[arrow] (FC3210) -- (tanh20);
\draw[arrow] (tanh20) -- (FC130);
\draw[arrow] (FC130) -- (nonstationarytransform40);\end{tikzpicture}
    \caption{$\text{DK}_\phi(t)$}
  \label{fig:mocap-dk-arch}
  \end{subfigure}\hfill
  \begin{subfigure}[b]{0.45\textwidth}
  \centering
    \begin{tikzpicture}[
    layer/.style={
        rectangle,
        draw=black,
        very thick,
        minimum width=2cm,
        align=center,
        rounded corners
    },
    arrow/.style={
        ->,
        >=latex,
        very thick
    }
]

\node[layer,fill={rgb,1:red,0.5529411764705883;green,0.8274509803921568;blue,0.7803921568627451},draw={rgb,1:red,0.49411764705882355;green,0.7411764705882353;blue,0.7019607843137254}] (x00) at (-1.15,0.0) {$x$};
\node[layer,fill={rgb,1:red,1.0;green,1.0;blue,0.7019607843137254},draw={rgb,1:red,0.8980392156862745;green,0.8980392156862745;blue,0.6313725490196078}] (FC3010) at (-1.15,-1.0) {FC-30};
\node[layer,fill={rgb,1:red,1.0;green,1.0;blue,0.7019607843137254},draw={rgb,1:red,0.8980392156862745;green,0.8980392156862745;blue,0.6313725490196078}] (tanh20) at (-1.15,-2.0) {tanh};
\node[layer,fill={rgb,1:red,1.0;green,1.0;blue,0.7019607843137254},draw={rgb,1:red,0.8980392156862745;green,0.8980392156862745;blue,0.6313725490196078}] (FC630) at (-1.15,-3.0) {FC-6};
\node[layer,fill={rgb,1:red,0.7450980392156863;green,0.7294117647058823;blue,0.8549019607843137},draw={rgb,1:red,0.6705882352941176;green,0.6549019607843137;blue,0.7686274509803922}] (drift40) at (-1.15,-4.0) {drift};
\draw[arrow] (x00) -- (FC3010);
\draw[arrow] (FC3010) -- (tanh20);
\draw[arrow] (tanh20) -- (FC630);
\draw[arrow] (FC630) -- (drift40);\end{tikzpicture}
    \caption{$f_\theta(z, t)$}
  \label{fig:mocap-drift-arch}
  \end{subfigure}
  \caption{
    Architecture diagrams for the encoder, 
    decoder, deep kernel and drift function used in the 
    MOCAP benchmark.
    We used very similar architectures to~\cite{yildiz_ode2vae_2019, li_scalable_2020}.
  Here $\mu(t)$ indicates the mean of the likelihood, $p_\theta(x\mid z(t))$.}
\label{fig:mocap-architecture}
\end{figure}
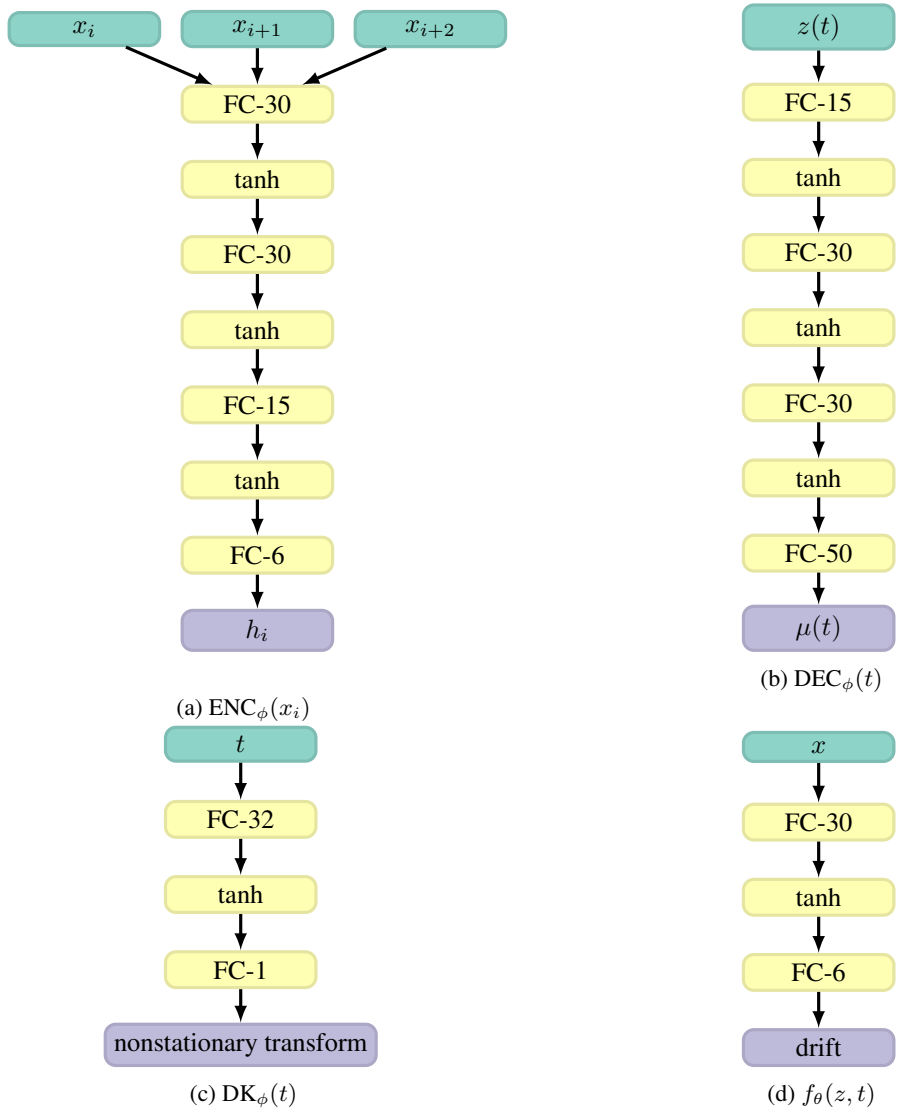

In terms of hyperparameters, we chose a batch size of 512.
We used a linear schedule to update the weighting on the 
KL-divergence from $0$ and $1$ over the course of $200$ iterations.
We make use of the nested Monte-Carlo scheme 
in equation~\eqref{eq:nested-monte-carlo} with $R=10$ and $S=10$.
We chose a learning rate of $0.01$ and decayed the learning rate 
each iteration according to the schedule $lr=\gamma lr$ with $\gamma = 
       \exp(\log(0.9) / 1000)$.
In terms of kernel parameters, we initialize $\sigma_n=10^{-5}$, 
$\ell=10^{-2}$, and $\sigma_f = 1$.
In terms of the prior on the diffusion term
we initialize ${\mu}_i = \sigma_i = 10^{-5}$ and set $\tilde{\mu}_i=\tilde{\sigma}_i = 1$.
In terms of the prior on the variance of the observations,
we initialize ${\mu}_i = \sigma_i = 10^{-2}$ and set $\tilde{\mu}_i=\tilde{\sigma}_i = 1$.
We considered both Softplus~\cite{li_scalable_2020} and tanh~\cite{yildiz_ode2vae_2019}
nonlinearities and found that tanh nonlinearities provided improved validation performance.
We train for $100$ epochs testing validation accuracy every 10 epochs. We report
the average test accuracy after training 10 models from different random seeds.

Previous studies tended to report mean-squared-error as,
$\text{MSE}\pm\text{ERROR}$.
We report $\text{RMSE}\pm\text{NEW ERROR}$ so that error units are consistent with the units of 
the original dataset.
To convert MSE to RMSE we used a first-order 
Taylor-series approximation,
\begin{equation}
  \begin{split}
    \text{RMSE} &= \sqrt{\text{MSE}}\\
    \text{NEW ERROR} &= \frac{1}{2} \text{ERROR}/ \text{RMSE}
  \end{split}
\end{equation}

\subsection{Neural SDE from video} \label{app:arch-nsde-video}
In this section, we provide a more detailed description 
on the numerical study described in Section~\ref{sec:nsde-from-video}.
We generated data by simulating a nonlinear pendulum with 
the equations,
\begin{equation}
  \begin{split}
    \dot{x} &= p \\
    \dot{p} &= - \sin(x),
  \end{split}
\end{equation}
for $30$ seconds while sampling the state at a frequency of 15Hz.
The architecture we used for this experiment is provided 
in Figure~\ref{fig:nsde-architecture}.
  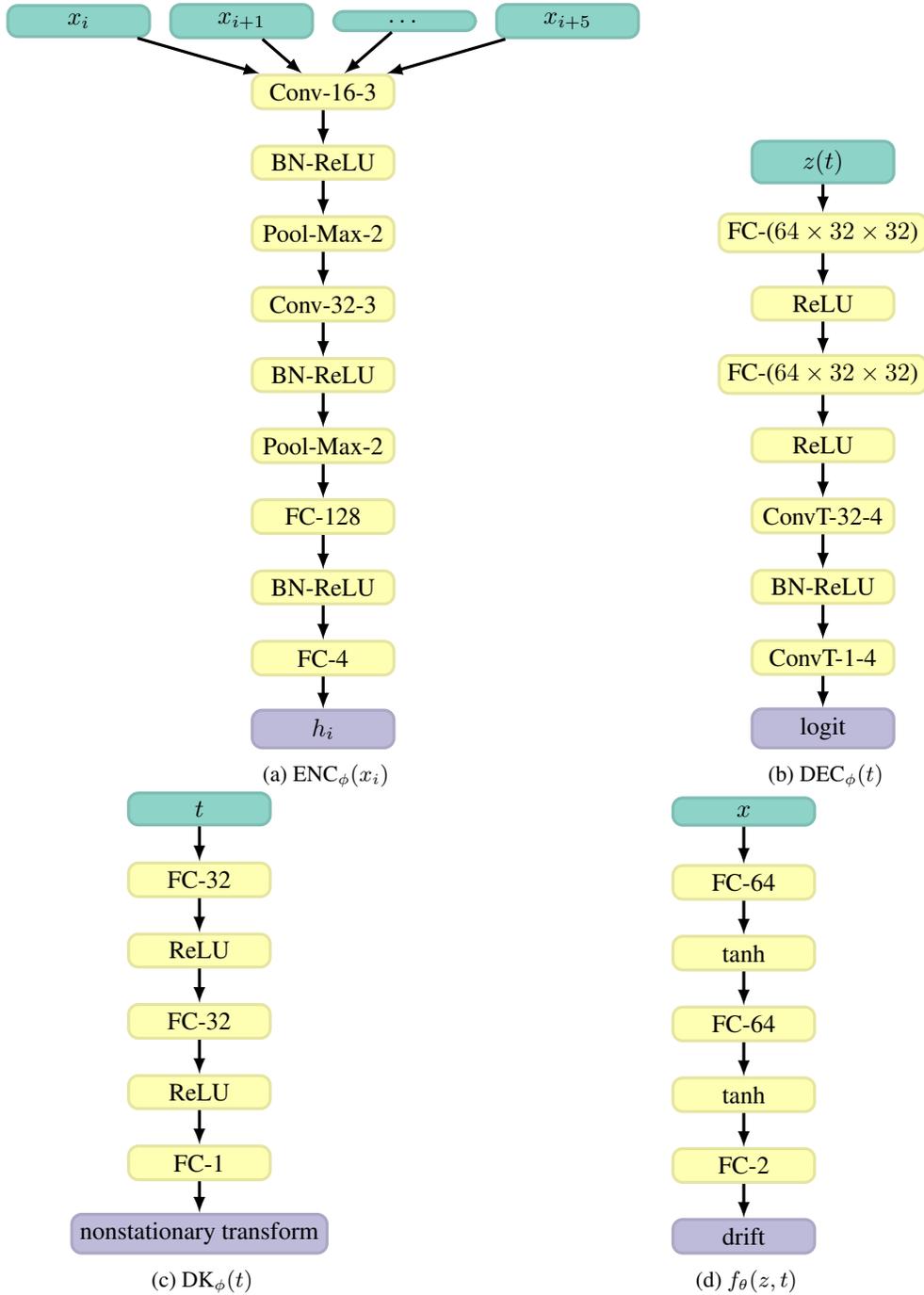
\begin{figure}[hp]
      \centering
  \begin{subfigure}[t]{0.7\textwidth}
  \centering
    \begin{tikzpicture}[
    layer/.style={
        rectangle,
        draw=black,
        very thick,
        minimum width=2cm,
        align=center,
        rounded corners
    },
    arrow/.style={
        ->,
        >=latex,
        very thick
    }
]

\node[layer,fill={rgb,1:red,0.5529411764705883;green,0.8274509803921568;blue,0.7803921568627451},draw={rgb,1:red,0.49411764705882355;green,0.7411764705882353;blue,0.7019607843137254}] (x_i00) at (-4.6,0.0) {$x_i$};
\node[layer,fill={rgb,1:red,0.5529411764705883;green,0.8274509803921568;blue,0.7803921568627451},draw={rgb,1:red,0.49411764705882355;green,0.7411764705882353;blue,0.7019607843137254}] (x_i101) at (-2.3,0.0) {$x_{i+1}$};
\node[layer,fill={rgb,1:red,0.5529411764705883;green,0.8274509803921568;blue,0.7803921568627451},draw={rgb,1:red,0.49411764705882355;green,0.7411764705882353;blue,0.7019607843137254}] (dots02) at (0.0,0.0) {$\dots$};
\node[layer,fill={rgb,1:red,0.5529411764705883;green,0.8274509803921568;blue,0.7803921568627451},draw={rgb,1:red,0.49411764705882355;green,0.7411764705882353;blue,0.7019607843137254}] (x_i503) at (2.3,0.0) {$x_{i+5}$};
\node[layer,fill={rgb,1:red,1.0;green,1.0;blue,0.7019607843137254},draw={rgb,1:red,0.8980392156862745;green,0.8980392156862745;blue,0.6313725490196078}] (Conv16310) at (-1.15,-1.0) {Conv-16-3};
\node[layer,fill={rgb,1:red,1.0;green,1.0;blue,0.7019607843137254},draw={rgb,1:red,0.8980392156862745;green,0.8980392156862745;blue,0.6313725490196078}] (BNReLU20) at (-1.15,-2.0) {BN-ReLU};
\node[layer,fill={rgb,1:red,1.0;green,1.0;blue,0.7019607843137254},draw={rgb,1:red,0.8980392156862745;green,0.8980392156862745;blue,0.6313725490196078}] (PoolMax230) at (-1.15,-3.0) {Pool-Max-2};
\node[layer,fill={rgb,1:red,1.0;green,1.0;blue,0.7019607843137254},draw={rgb,1:red,0.8980392156862745;green,0.8980392156862745;blue,0.6313725490196078}] (Conv32340) at (-1.15,-4.0) {Conv-32-3};
\node[layer,fill={rgb,1:red,1.0;green,1.0;blue,0.7019607843137254},draw={rgb,1:red,0.8980392156862745;green,0.8980392156862745;blue,0.6313725490196078}] (BNReLU50) at (-1.15,-5.0) {BN-ReLU};
\node[layer,fill={rgb,1:red,1.0;green,1.0;blue,0.7019607843137254},draw={rgb,1:red,0.8980392156862745;green,0.8980392156862745;blue,0.6313725490196078}] (PoolMax260) at (-1.15,-6.0) {Pool-Max-2};
\node[layer,fill={rgb,1:red,1.0;green,1.0;blue,0.7019607843137254},draw={rgb,1:red,0.8980392156862745;green,0.8980392156862745;blue,0.6313725490196078}] (FC12870) at (-1.15,-7.0) {FC-128};
\node[layer,fill={rgb,1:red,1.0;green,1.0;blue,0.7019607843137254},draw={rgb,1:red,0.8980392156862745;green,0.8980392156862745;blue,0.6313725490196078}] (BNReLU80) at (-1.15,-8.0) {BN-ReLU};
\node[layer,fill={rgb,1:red,1.0;green,1.0;blue,0.7019607843137254},draw={rgb,1:red,0.8980392156862745;green,0.8980392156862745;blue,0.6313725490196078}] (FC490) at (-1.15,-9.0) {FC-4};
\node[layer,fill={rgb,1:red,0.7450980392156863;green,0.7294117647058823;blue,0.8549019607843137},draw={rgb,1:red,0.6705882352941176;green,0.6549019607843137;blue,0.7686274509803922}] (h_i100) at (-1.15,-10.0) {$h_i$};
\draw[arrow] (x_i00) -- (Conv16310);
\draw[arrow] (x_i101) -- (Conv16310);
\draw[arrow] (dots02) -- (Conv16310);
\draw[arrow] (x_i503) -- (Conv16310);
\draw[arrow] (Conv16310) -- (BNReLU20);
\draw[arrow] (BNReLU20) -- (PoolMax230);
\draw[arrow] (PoolMax230) -- (Conv32340);
\draw[arrow] (Conv32340) -- (BNReLU50);
\draw[arrow] (BNReLU50) -- (PoolMax260);
\draw[arrow] (PoolMax260) -- (FC12870);
\draw[arrow] (FC12870) -- (BNReLU80);
\draw[arrow] (BNReLU80) -- (FC490);
\draw[arrow] (FC490) -- (h_i100);\end{tikzpicture}
    \caption{$\text{ENC}_\phi(x_i)$}
  \label{fig:nsde-enc-arch}
  \end{subfigure}\hfill
  \begin{subfigure}[t]{0.29\textwidth}
  \centering
    \begin{tikzpicture}[
    layer/.style={
        rectangle,
        draw=black,
        very thick,
        minimum width=2cm,
        align=center,
        rounded corners
    },
    arrow/.style={
        ->,
        >=latex,
        very thick
    }
]

\node[layer,fill={rgb,1:red,0.5529411764705883;green,0.8274509803921568;blue,0.7803921568627451},draw={rgb,1:red,0.49411764705882355;green,0.7411764705882353;blue,0.7019607843137254}] (zt00) at (-1.15,0.0) {$z(t)$};
\node[layer,fill={rgb,1:red,1.0;green,1.0;blue,0.7019607843137254},draw={rgb,1:red,0.8980392156862745;green,0.8980392156862745;blue,0.6313725490196078}] (FC64times32times3210) at (-1.15,-1.0) {FC-($64\times 32 \times 32$)};
\node[layer,fill={rgb,1:red,1.0;green,1.0;blue,0.7019607843137254},draw={rgb,1:red,0.8980392156862745;green,0.8980392156862745;blue,0.6313725490196078}] (ReLU20) at (-1.15,-2.0) {ReLU};
\node[layer,fill={rgb,1:red,1.0;green,1.0;blue,0.7019607843137254},draw={rgb,1:red,0.8980392156862745;green,0.8980392156862745;blue,0.6313725490196078}] (FC64times32times3230) at (-1.15,-3.0) {FC-($64\times 32 \times 32$)};
\node[layer,fill={rgb,1:red,1.0;green,1.0;blue,0.7019607843137254},draw={rgb,1:red,0.8980392156862745;green,0.8980392156862745;blue,0.6313725490196078}] (ReLU40) at (-1.15,-4.0) {ReLU};
\node[layer,fill={rgb,1:red,1.0;green,1.0;blue,0.7019607843137254},draw={rgb,1:red,0.8980392156862745;green,0.8980392156862745;blue,0.6313725490196078}] (ConvT32450) at (-1.15,-5.0) {ConvT-32-4};
\node[layer,fill={rgb,1:red,1.0;green,1.0;blue,0.7019607843137254},draw={rgb,1:red,0.8980392156862745;green,0.8980392156862745;blue,0.6313725490196078}] (BNReLU60) at (-1.15,-6.0) {BN-ReLU};
\node[layer,fill={rgb,1:red,1.0;green,1.0;blue,0.7019607843137254},draw={rgb,1:red,0.8980392156862745;green,0.8980392156862745;blue,0.6313725490196078}] (ConvT1470) at (-1.15,-7.0) {ConvT-1-4};
\node[layer,fill={rgb,1:red,0.7450980392156863;green,0.7294117647058823;blue,0.8549019607843137},draw={rgb,1:red,0.6705882352941176;green,0.6549019607843137;blue,0.7686274509803922}] (logit80) at (-1.15,-8.0) {logit};
\draw[arrow] (zt00) -- (FC64times32times3210);
\draw[arrow] (FC64times32times3210) -- (ReLU20);
\draw[arrow] (ReLU20) -- (FC64times32times3230);
\draw[arrow] (FC64times32times3230) -- (ReLU40);
\draw[arrow] (ReLU40) -- (ConvT32450);
\draw[arrow] (ConvT32450) -- (BNReLU60);
\draw[arrow] (BNReLU60) -- (ConvT1470);
\draw[arrow] (ConvT1470) -- (logit80);\end{tikzpicture}
    \caption{$\text{DEC}_\phi(t)$}
  \label{fig:nsde-dec-arch}
  \end{subfigure}\\
  \begin{subfigure}[b]{0.45\textwidth}
  \centering
    \begin{tikzpicture}[
    layer/.style={
        rectangle,
        draw=black,
        very thick,
        minimum width=2cm,
        align=center,
        rounded corners
    },
    arrow/.style={
        ->,
        >=latex,
        very thick
    }
]

\node[layer,fill={rgb,1:red,0.5529411764705883;green,0.8274509803921568;blue,0.7803921568627451},draw={rgb,1:red,0.49411764705882355;green,0.7411764705882353;blue,0.7019607843137254}] (t00) at (-1.15,0.0) {$t$};
\node[layer,fill={rgb,1:red,1.0;green,1.0;blue,0.7019607843137254},draw={rgb,1:red,0.8980392156862745;green,0.8980392156862745;blue,0.6313725490196078}] (FC3210) at (-1.15,-1.0) {FC-32};
\node[layer,fill={rgb,1:red,1.0;green,1.0;blue,0.7019607843137254},draw={rgb,1:red,0.8980392156862745;green,0.8980392156862745;blue,0.6313725490196078}] (ReLU20) at (-1.15,-2.0) {ReLU};
\node[layer,fill={rgb,1:red,1.0;green,1.0;blue,0.7019607843137254},draw={rgb,1:red,0.8980392156862745;green,0.8980392156862745;blue,0.6313725490196078}] (FC3230) at (-1.15,-3.0) {FC-32};
\node[layer,fill={rgb,1:red,1.0;green,1.0;blue,0.7019607843137254},draw={rgb,1:red,0.8980392156862745;green,0.8980392156862745;blue,0.6313725490196078}] (ReLU40) at (-1.15,-4.0) {ReLU};
\node[layer,fill={rgb,1:red,1.0;green,1.0;blue,0.7019607843137254},draw={rgb,1:red,0.8980392156862745;green,0.8980392156862745;blue,0.6313725490196078}] (FC150) at (-1.15,-5.0) {FC-1};
\node[layer,fill={rgb,1:red,0.7450980392156863;green,0.7294117647058823;blue,0.8549019607843137},draw={rgb,1:red,0.6705882352941176;green,0.6549019607843137;blue,0.7686274509803922}] (nonstationarytransform60) at (-1.15,-6.0) {nonstationary transform};
\draw[arrow] (t00) -- (FC3210);
\draw[arrow] (FC3210) -- (ReLU20);
\draw[arrow] (ReLU20) -- (FC3230);
\draw[arrow] (FC3230) -- (ReLU40);
\draw[arrow] (ReLU40) -- (FC150);
\draw[arrow] (FC150) -- (nonstationarytransform60);\end{tikzpicture}
    \caption{$\text{DK}_\phi(t)$}
  \label{fig:nsde-dk-arch}
  \end{subfigure}\hfill
  \begin{subfigure}[b]{0.45\textwidth}
  \centering
    \begin{tikzpicture}[
    layer/.style={
        rectangle,
        draw=black,
        very thick,
        minimum width=2cm,
        align=center,
        rounded corners
    },
    arrow/.style={
        ->,
        >=latex,
        very thick
    }
]

\node[layer,fill={rgb,1:red,0.5529411764705883;green,0.8274509803921568;blue,0.7803921568627451},draw={rgb,1:red,0.49411764705882355;green,0.7411764705882353;blue,0.7019607843137254}] (x00) at (-1.15,0.0) {$x$};
\node[layer,fill={rgb,1:red,1.0;green,1.0;blue,0.7019607843137254},draw={rgb,1:red,0.8980392156862745;green,0.8980392156862745;blue,0.6313725490196078}] (FC6410) at (-1.15,-1.0) {FC-64};
\node[layer,fill={rgb,1:red,1.0;green,1.0;blue,0.7019607843137254},draw={rgb,1:red,0.8980392156862745;green,0.8980392156862745;blue,0.6313725490196078}] (tanh20) at (-1.15,-2.0) {tanh};
\node[layer,fill={rgb,1:red,1.0;green,1.0;blue,0.7019607843137254},draw={rgb,1:red,0.8980392156862745;green,0.8980392156862745;blue,0.6313725490196078}] (FC6430) at (-1.15,-3.0) {FC-64};
\node[layer,fill={rgb,1:red,1.0;green,1.0;blue,0.7019607843137254},draw={rgb,1:red,0.8980392156862745;green,0.8980392156862745;blue,0.6313725490196078}] (tanh40) at (-1.15,-4.0) {tanh};
\node[layer,fill={rgb,1:red,1.0;green,1.0;blue,0.7019607843137254},draw={rgb,1:red,0.8980392156862745;green,0.8980392156862745;blue,0.6313725490196078}] (FC250) at (-1.15,-5.0) {FC-2};
\node[layer,fill={rgb,1:red,0.7450980392156863;green,0.7294117647058823;blue,0.8549019607843137},draw={rgb,1:red,0.6705882352941176;green,0.6549019607843137;blue,0.7686274509803922}] (drift60) at (-1.15,-6.0) {drift};
\draw[arrow] (x00) -- (FC6410);
\draw[arrow] (FC6410) -- (tanh20);
\draw[arrow] (tanh20) -- (FC6430);
\draw[arrow] (FC6430) -- (tanh40);
\draw[arrow] (tanh40) -- (FC250);
\draw[arrow] (FC250) -- (drift60);\end{tikzpicture}
    \caption{$f_\theta(z, t)$}
  \label{fig:nsde-drift-arch}
  \end{subfigure}
  \caption{
    Architecture diagrams for the encoder, 
    decoder, deep kernel and drift function used in the 
    Neural SDE from video example.}
\label{fig:nsde-architecture}
\end{figure}

In terms of hyperparameters, we chose a batch size of 128.
We gradually increased the weighting of the KL-divergence 
term using a linear schedule over the course of $1000$ iterations.
We used the nested Monte-Carlo method suggested in Equation~\ref{eq:nested-monte-carlo}
and set $R=20$ and $S=10$.
We chose a learning rate of $0.001$ and decayed the learning rate 
each iteration according to the schedule $lr=\gamma lr$ with $\gamma = 
\exp(\log(0.9) / 1000)$.
In terms of kernel parameters, we initialized $\sigma_f=1$, 
$\sigma_n = 10^{-5}$, and $\ell=10^{-2}$.
We placed a log-normal prior on the diffusion term and approximated 
the posterior using log-normal variational distribution. 
With regards to the prior on the diffusion term
we initialize ${\mu}_i = \sigma_i = 0.1$ and set $\tilde{\mu}_i=\tilde{\sigma}_i = 10^{-5}$.

\section{Numerical study on the effect of Monte-Carlo parameters} \label{app:study-on-mc}
In this study we investigate the effect of varying the number of nested Monte 
Carlo samples on the rate of validation RMSE convergence.
To do so, we consider the problem of building a predictive model for a four-dimensional 
predator-prey system. The system consists of two predators and two prey where 
there is a competitive dynamic between the predators.
The equations governing the dynamics of the system 
are, 
\begin{align}
    \dot{x}_1 &= x_1 \left(\alpha_1 - \beta_1y_1 - \gamma_1 y_2\right) \left(1 - \frac{x_1}{k_1}\right),\\
    \dot{x}_2 &= x_2 \left(\alpha_2 - \beta_2y_1 - \gamma_2 y_2\right) \left(1 - \frac{x_2}{k_2}\right),\\
    \dot{y}_1 &= y_1 \left(-\delta_1 + \epsilon_1 x_1 + \xi_1 x_2 - \nu_1 y_2 \right), \\
    \dot{y}_2 &= y_2 \left(-\delta_2 + \epsilon_2 x_1 + \xi_2 x_2 +\nu_2 y_1 \right),
\end{align}
where $\alpha_i$ is the grow rate of the prey $x_i$, $\beta_i$ is the rate that predator 
$y_1$ is consuming prey $x_i$, $\gamma_i$ is the rate that predator $y_2$ is consuming
prey $x_i$, $k_i$ is the carrying capacity for prey $x_i$, $\delta_i$ is the 
death rate of predator $y_i$, $\epsilon_i$ is the conversion rate for predator 
$y_i$ from $x_1$, $\xi_i$ is the conversion rate for predator $y_i$ from prey 
$x_2$, and $\nu_i$ represents the competitive effects on $y_i$ caused by the other 
predator.

We simulated the system for 300 units of time collecting data at a frequency 
of 10Hz. We assume a Gaussian noise with a standard deviation of $10^{-2}$.
For all experiments we used the following hyperparameters:
a batch size of $256$, $1000$ warmup iterations, and a learning rate of
$10^{-3}$.
    In terms of kernel parameters, we initialized $\sigma_f = 1$, 
    $\ell=10^{-2}$, and $\sigma_n = 10^{-5}$.
    For the diffusion term we set $\mu_i = \sigma_i = 10^{-5}$ and 
    $\tilde{\mu}_i = \tilde{\sigma}_i = 1$.
    The architecture description for the neural networks used in this example
    are provided in Figure~\ref{fig:gradvar-arch-diagram}.

Results are summarized in Figure~\ref{fig:variance-of-mc} below. 
On this problem we find we are able to achieve a reasonable
    validation accuracy for $S=10$, $S=50$, and $S=100$; however, it is
    challenging to know beforehand what gradient variance will be acceptable for
    a particular data set. We see that increasing $S$ increases the total number
    of function evaluations. Note that the total number of parallel evaluations of
    the forward model in all cases remains constant.

\begin{figure}[htb]
  \centering
\begin{subfigure}{.5\textwidth}
  \centering
  \includegraphics[width=1.\linewidth]{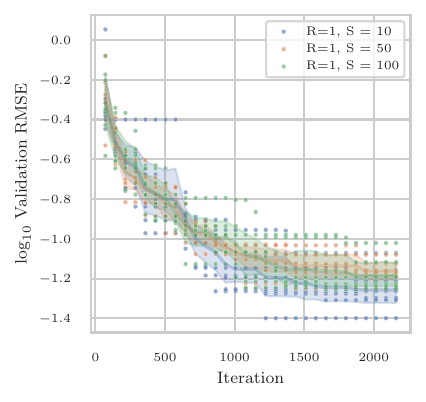}
\end{subfigure}\hfill
\begin{subfigure}{.5\textwidth}
  \centering
  \includegraphics[width=1.\linewidth]{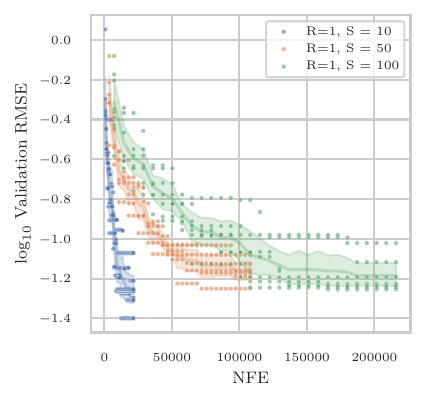}
\end{subfigure}
    \caption{Validation RMSE versus iteration (L) and number of function
    evaluations (R). }
    \label{fig:variance-of-mc}
\end{figure}

  \begin{figure}[hp]
  \begin{subfigure}[t]{0.7\textwidth}
  \centering
    \begin{tikzpicture}[
    layer/.style={
        rectangle,
        draw=black,
        very thick,
        minimum width=2cm,
        align=center,
        rounded corners
    },
    arrow/.style={
        ->,
        >=latex,
        very thick
    }
]

\node[layer,fill={rgb,1:red,0.5529411764705883;green,0.8274509803921568;blue,0.7803921568627451},draw={rgb,1:red,0.49411764705882355;green,0.7411764705882353;blue,0.7019607843137254}] (x_i00) at (-1.15,0.0) {$x_i$};
\node[layer,fill={rgb,1:red,1.0;green,1.0;blue,0.7019607843137254},draw={rgb,1:red,0.8980392156862745;green,0.8980392156862745;blue,0.6313725490196078}] (FC3210) at (-1.15,-1.0) {FC-32};
\node[layer,fill={rgb,1:red,1.0;green,1.0;blue,0.7019607843137254},draw={rgb,1:red,0.8980392156862745;green,0.8980392156862745;blue,0.6313725490196078}] (ReLU20) at (-1.15,-2.0) {ReLU};
\node[layer,fill={rgb,1:red,1.0;green,1.0;blue,0.7019607843137254},draw={rgb,1:red,0.8980392156862745;green,0.8980392156862745;blue,0.6313725490196078}] (FC3230) at (-1.15,-3.0) {FC-32};
\node[layer,fill={rgb,1:red,1.0;green,1.0;blue,0.7019607843137254},draw={rgb,1:red,0.8980392156862745;green,0.8980392156862745;blue,0.6313725490196078}] (ReLU40) at (-1.15,-4.0) {ReLU};
\node[layer,fill={rgb,1:red,1.0;green,1.0;blue,0.7019607843137254},draw={rgb,1:red,0.8980392156862745;green,0.8980392156862745;blue,0.6313725490196078}] (FC450) at (-1.15,-5.0) {FC-4};
\node[layer,fill={rgb,1:red,0.7450980392156863;green,0.7294117647058823;blue,0.8549019607843137},draw={rgb,1:red,0.6705882352941176;green,0.6549019607843137;blue,0.7686274509803922}] (mt_i60) at (-2.3,-6.0) {$m(t_i)$};
\node[layer,fill={rgb,1:red,0.7450980392156863;green,0.7294117647058823;blue,0.8549019607843137},draw={rgb,1:red,0.6705882352941176;green,0.6549019607843137;blue,0.7686274509803922}] (logSt_i61) at (0.0,-6.0) {$\log S(t_i)$};
\draw[arrow] (x_i00) -- (FC3210);
\draw[arrow] (FC3210) -- (ReLU20);
\draw[arrow] (ReLU20) -- (FC3230);
\draw[arrow] (FC3230) -- (ReLU40);
\draw[arrow] (ReLU40) -- (FC450);
\draw[arrow] (FC450) -- (mt_i60);
\draw[arrow] (FC450) -- (logSt_i61);\draw[arrow] (x_i00) to [out=200,in=110] (mt_i60);\end{tikzpicture}
    \caption{$\text{ENC}_\phi(x_i)$}
  \label{fig:gradvar-enc-arch}
  \end{subfigure}\hfill
  \begin{subfigure}[b]{0.45\textwidth}
  \centering
    \begin{tikzpicture}[
    layer/.style={
        rectangle,
        draw=black,
        very thick,
        minimum width=2cm,
        align=center,
        rounded corners
    },
    arrow/.style={
        ->,
        >=latex,
        very thick
    }
]

\node[layer,fill={rgb,1:red,0.5529411764705883;green,0.8274509803921568;blue,0.7803921568627451},draw={rgb,1:red,0.49411764705882355;green,0.7411764705882353;blue,0.7019607843137254}] (t00) at (-1.15,0.0) {$t$};
\node[layer,fill={rgb,1:red,1.0;green,1.0;blue,0.7019607843137254},draw={rgb,1:red,0.8980392156862745;green,0.8980392156862745;blue,0.6313725490196078}] (FC3210) at (-1.15,-1.0) {FC-32};
\node[layer,fill={rgb,1:red,1.0;green,1.0;blue,0.7019607843137254},draw={rgb,1:red,0.8980392156862745;green,0.8980392156862745;blue,0.6313725490196078}] (ReLU20) at (-1.15,-2.0) {ReLU};
\node[layer,fill={rgb,1:red,1.0;green,1.0;blue,0.7019607843137254},draw={rgb,1:red,0.8980392156862745;green,0.8980392156862745;blue,0.6313725490196078}] (FC3230) at (-1.15,-3.0) {FC-32};
\node[layer,fill={rgb,1:red,1.0;green,1.0;blue,0.7019607843137254},draw={rgb,1:red,0.8980392156862745;green,0.8980392156862745;blue,0.6313725490196078}] (ReLU40) at (-1.15,-4.0) {ReLU};
\node[layer,fill={rgb,1:red,1.0;green,1.0;blue,0.7019607843137254},draw={rgb,1:red,0.8980392156862745;green,0.8980392156862745;blue,0.6313725490196078}] (FC150) at (-1.15,-5.0) {FC-1};
\node[layer,fill={rgb,1:red,0.7450980392156863;green,0.7294117647058823;blue,0.8549019607843137},draw={rgb,1:red,0.6705882352941176;green,0.6549019607843137;blue,0.7686274509803922}] (nonstationarytransform60) at (-1.15,-6.0) {nonstationary transform};
\draw[arrow] (t00) -- (FC3210);
\draw[arrow] (FC3210) -- (ReLU20);
\draw[arrow] (ReLU20) -- (FC3230);
\draw[arrow] (FC3230) -- (ReLU40);
\draw[arrow] (ReLU40) -- (FC150);
\draw[arrow] (FC150) -- (nonstationarytransform60);\end{tikzpicture}
    \caption{$\text{DK}_\phi(t)$}
  \label{fig:gradvar-dk-arch}
  \end{subfigure}\hfill
  \begin{subfigure}[b]{0.45\textwidth}
  \centering
    \begin{tikzpicture}[
    layer/.style={
        rectangle,
        draw=black,
        very thick,
        minimum width=2cm,
        align=center,
        rounded corners
    },
    arrow/.style={
        ->,
        >=latex,
        very thick
    }
]

\node[layer,fill={rgb,1:red,0.5529411764705883;green,0.8274509803921568;blue,0.7803921568627451},draw={rgb,1:red,0.49411764705882355;green,0.7411764705882353;blue,0.7019607843137254}] (x00) at (-1.15,0.0) {$x$};
\node[layer,fill={rgb,1:red,1.0;green,1.0;blue,0.7019607843137254},draw={rgb,1:red,0.8980392156862745;green,0.8980392156862745;blue,0.6313725490196078}] (FC6410) at (-1.15,-1.0) {FC-64};
\node[layer,fill={rgb,1:red,1.0;green,1.0;blue,0.7019607843137254},draw={rgb,1:red,0.8980392156862745;green,0.8980392156862745;blue,0.6313725490196078}] (ReLU20) at (-1.15,-2.0) {ReLU};
\node[layer,fill={rgb,1:red,1.0;green,1.0;blue,0.7019607843137254},draw={rgb,1:red,0.8980392156862745;green,0.8980392156862745;blue,0.6313725490196078}] (FC6430) at (-1.15,-3.0) {FC-64};
\node[layer,fill={rgb,1:red,1.0;green,1.0;blue,0.7019607843137254},draw={rgb,1:red,0.8980392156862745;green,0.8980392156862745;blue,0.6313725490196078}] (ReLU40) at (-1.15,-4.0) {ReLU};
\node[layer,fill={rgb,1:red,1.0;green,1.0;blue,0.7019607843137254},draw={rgb,1:red,0.8980392156862745;green,0.8980392156862745;blue,0.6313725490196078}] (FC250) at (-1.15,-5.0) {FC-2};
\node[layer,fill={rgb,1:red,0.7450980392156863;green,0.7294117647058823;blue,0.8549019607843137},draw={rgb,1:red,0.6705882352941176;green,0.6549019607843137;blue,0.7686274509803922}] (driftfunction60) at (-1.15,-6.0) {drift function};
\draw[arrow] (x00) -- (FC6410);
\draw[arrow] (FC6410) -- (ReLU20);
\draw[arrow] (ReLU20) -- (FC6430);
\draw[arrow] (FC6430) -- (ReLU40);
\draw[arrow] (ReLU40) -- (FC250);
\draw[arrow] (FC250) -- (driftfunction60);\end{tikzpicture}
    \caption{$f_\theta(z, t)$}
  \label{fig:gradvar-drift-arch}
  \end{subfigure}
  \caption{
    Architecture diagrams for the encoder, 
    deep kernel and drift function used in the 
    Monte-Carlo parameters study.}
\label{fig:gradvar-arch-diagram}
\end{figure}
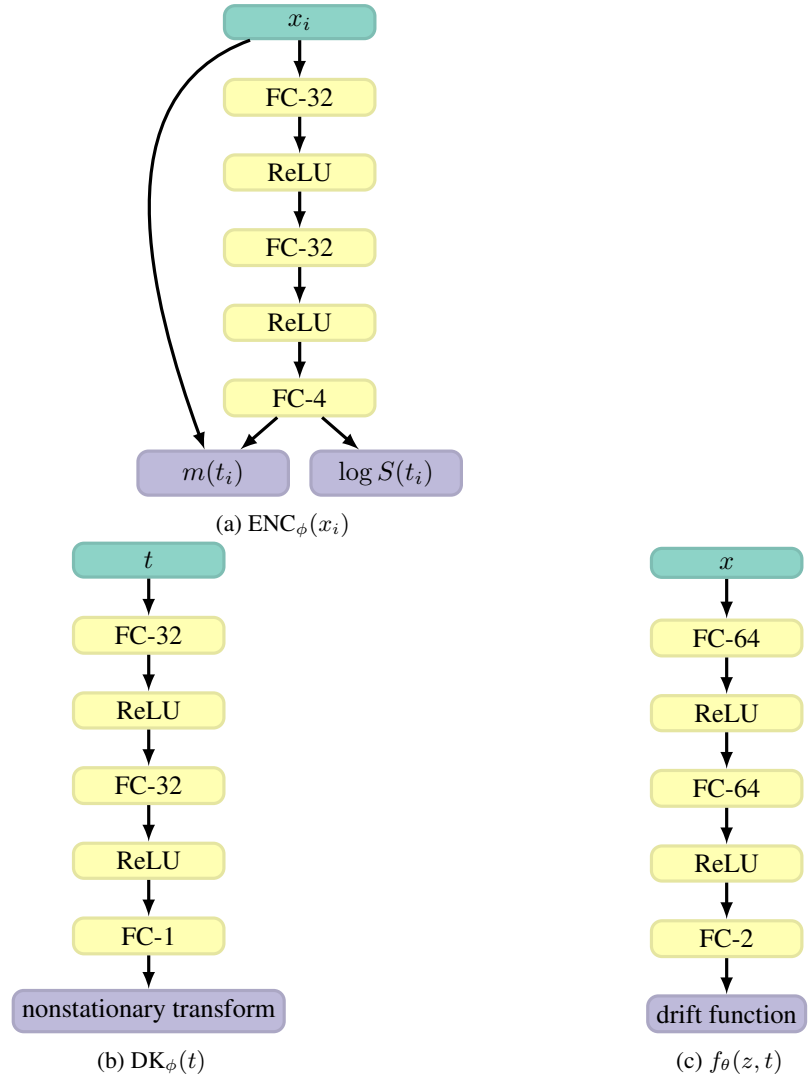

\end{document}